\documentclass{article}

    \PassOptionsToPackage{numbers, compress}{natbib}

\usepackage{tikz}

     \usepackage[final]{neurips_2019}

\usepackage[utf8]{inputenc} 
\usepackage[T1]{fontenc}    
\usepackage{hyperref}       
\usepackage{url}            
\usepackage{booktabs}       
\usepackage{amsfonts}       
\usepackage{nicefrac}       
\usepackage{microtype}      

\usepackage{times}
\usepackage{soul}
\usepackage{lipsum}  
\usepackage{caption}
\usepackage{graphicx}
\usepackage{amsmath,amssymb,amsfonts,amsthm}
\usepackage{booktabs}
\usepackage{multirow}
\usepackage[ruled,linesnumbered,noend]{algorithm2e}
\usepackage{wrapfig}

\newcommand{\squishlisttwo}{
 \begin{list}{$\bullet$}
  { \setlength{\itemsep}{0pt}
     \setlength{\parsep}{0pt}
    \setlength{\topsep}{0pt}
    \setlength{\partopsep}{0pt}
    \setlength{\leftmargin}{1em}
    \setlength{\labelwidth}{1.5em}
    \setlength{\labelsep}{0.5em} } }

\newcommand{\squishend}{
  \end{list}  }

\definecolor{red}{rgb}{0,0,0}	

\newtheorem{prop}{Proposition}
\theoremstyle{remark}
\newtheorem{remark}{Remark}

\title{Implicit Posterior Variational Inference for\\
Deep Gaussian Processes}

\author{Haibin Yu\thanks{Equal contribution.}\ \ , 
  Yizhou Chen$^*$, 
  Zhongxiang Dai, 
  Bryan Kian Hsiang Low, \textbf{and}
  Patrick Jaillet$^\dag$
  \\
  Dept. of Computer Science, National University of Singapore, Republic of Singapore \\
  Dept. of Electrical Engineering and Computer Science, MIT, USA$^\dag$ \\
  \texttt{\{haibin,ychen041,daiz,lowkh\}@comp.nus.edu.sg}, \texttt{jaillet@mit.edu$^\dag$}\\
}

\begin{document}

\maketitle

\begin{abstract}
    A multi-layer \emph{deep Gaussian process} (DGP) model is a hierarchical composition of GP models with a greater expressive power. Exact DGP inference is intractable, which has motivated the recent development of deterministic and stochastic approximation methods. Unfortunately, the deterministic approximation methods yield a biased posterior belief while the stochastic one is computationally costly. This paper presents an \emph{implicit posterior variational inference} (IPVI) framework for DGPs that can ideally recover an unbiased posterior belief and still preserve time efficiency. Inspired by generative adversarial networks, our IPVI framework achieves this by casting the DGP inference problem as a two-player game in which a Nash equilibrium, interestingly, coincides with an unbiased posterior belief. This consequently inspires us to devise a best-response dynamics algorithm to search for a Nash equilibrium (i.e.,  an unbiased posterior belief). Empirical evaluation shows that IPVI outperforms the state-of-the-art approximation methods for DGPs.
\end{abstract}
%
\section{Introduction} 
\label{sec: introduction}
The expressive power of the Bayesian non-parametric \emph{Gaussian process} (GP)~\cite{Rasmussen06} models can be significantly boosted by composing them hierarchically into a multi-layer \emph{deep GP} (DGP) model, as shown in the seminal work of~\cite{damianou2013deep}.
Though the DGP model can likewise exploit the notion of inducing variables~\cite{LowUAI13,NghiaICML16,HoangICML16,LowDyDESS15,LowAAAI15,candela05,LowAAAI14,Haibin19} to improve its scalability, doing so does not immediately entail tractable inference, unlike the GP model.
This has motivated the development of deterministic and stochastic approximation methods, the former of which have imposed varying structural assumptions across the DGP hidden layers and assumed a Gaussian posterior belief of the inducing variables~\cite{bui2016deep,dai2015variational,damianou2013deep,hensman2014nested,salimbeni2017doubly}.
However, the work of~\cite{havasi2018inference} has demonstrated that with at least one DGP hidden layer, the posterior belief of the inducing variables is usually non-Gaussian, hence potentially compromising the performance of the deterministic approximation methods due to their biased posterior belief. To resolve this, the stochastic approximation method of~\cite{havasi2018inference} utilizes \emph{stochastic gradient Hamiltonian Monte Carlo} (SGHMC) sampling to draw unbiased samples from the posterior belief.
But, generating such samples is computationally costly in both training and prediction due to its sequential sampling procedure~\cite{wang2018adversarial} and its convergence is also difficult to assess.
So, the challenge remains in devising a time-efficient approximation method that can recover  an unbiased posterior belief.

This paper presents an \emph{implicit posterior variational inference} (IPVI) framework for DGPs (Section~\ref{sec: IPVI DGP}) that can ideally recover an unbiased posterior belief and still preserve time efficiency, hence combining the best of both worlds (respectively, stochastic and deterministic approximation methods).
Inspired by generative adversarial networks~\cite{goodfellow2014generative} that can generate samples to represent complex distributions which are hard to model using an explicit likelihood \cite{karras2017progressive,oord2016wavenet},  
our IPVI framework achieves this by casting the DGP inference problem as a two-player game in which a Nash equilibrium, interestingly, coincides with an unbiased posterior belief. 
This consequently inspires us to devise a best-response dynamics algorithm to search for a Nash equilibrium~\cite{awerbuch2008fast} (i.e., an unbiased posterior belief).
In Section~\ref{sec: arichitecture}, we discuss how the architecture of the generator 
in our IPVI framework is designed to enable parameter tying for a DGP model to alleviate overfitting.
We empirically evaluate the performance of IPVI on several real-world datasets in supervised 
(e.g., regression and classification) 
and unsupervised learning tasks (Section~\ref{sec: experiments}).

%
\section{Background and Related Work} 
\label{sec: background}
\textbf{Gaussian Process (GP).} 
    \label{subsec: single layer gaussian process}
Let a random function $f: \mathbb{R}^D \rightarrow \mathbb{R}$ be distributed by a GP with a zero prior mean and covariance function $k: \mathbb{R}^D \times \mathbb{R}^D \rightarrow \mathbb{R}$.
That is, suppose that a set $\mathbf{y} \triangleq \{ y_n\}_{n=1}^{N}$ of $N$ noisy observed outputs $y_n\triangleq f(\mathbf{x}_n)+\varepsilon(\mathbf{x}_n)$ (i.e., corrupted by an i.i.d. Gaussian noise $\varepsilon(\mathbf{x}_n)$ with noise variance $\nu^2$) are available for some set $\mathbf{X} \triangleq \{ \mathbf{x}_n\}_{n=1}^{N}$ of $N$ training inputs.
Then,  
the set $\mathbf{f} \triangleq \{f(\mathbf{x}_n)\}_{n=1}^{N}$ of latent outputs follow a Gaussian prior belief $p(\mathbf{f}) \triangleq \mathcal{N}(\mathbf{f}|\mathbf{0}, \mathbf{K}_{\mathbf{X}\mathbf{X}})$ where $\mathbf{K}_{\mathbf{X}\mathbf{X}}$ denotes a covariance matrix with components $k(\mathbf{x}_n, \mathbf{x}_{n'})$ for $n,n'=1,\ldots,N$.
It follows that $p(\mathbf{y}|\mathbf{f})=\mathcal{N}(\mathbf{y}|\mathbf{f},\nu^2\mathbf{I})$.
The GP predictive/posterior belief of the latent outputs $\mathbf{f}^{\star} \triangleq \{f(\mathbf{x}^{\star})\}_{\mathbf{x}^{\star} \in \mathbf{X}^{\star}}$ for any set $\mathbf{X}^{\star}$ of test inputs can be computed in closed form~\cite{Rasmussen06} by marginalizing out $\mathbf{f}$: 
$p(\mathbf{f}^{\star}|\mathbf{y}) = \int p(\mathbf{f}^{\star}|\mathbf{f})\ p(\mathbf{f}|\mathbf{y})\ \mathrm{d}\mathbf{f}$ but incurs cubic time in $N$, hence scaling poorly to massive datasets.

To improve its scalability to linear time in $N$, the \emph{sparse GP} (SGP) models spanned by the unifying view of~\cite{candela05} exploit a set $\mathbf{u} \triangleq \{u_m \triangleq f(\mathbf{z}_m)\}_{m=1}^{M}$ of inducing output variables for some small set $\mathbf{Z} \triangleq \{\mathbf{z}_m\}_{m=1}^{M}$ of inducing inputs (i.e., $M \ll N$).
Then,
\begin{equation}
p(\mathbf{y}, \mathbf{f}, \mathbf{u}) = p(\mathbf{y}|\mathbf{f})\ p(\mathbf{f}|\mathbf{u})\ p(\mathbf{u})
\label{crap}
\end{equation}
such that
$p(\mathbf{f}|\mathbf{u}) = \mathcal{N}(\mathbf{f}|\mathbf{K}_{\mathbf{X}\mathbf{Z}}\mathbf{K}_{\mathbf{Z}\mathbf{Z}}^{-1}\mathbf{u}, \mathbf{K}_{\mathbf{X}\mathbf{X}} - \mathbf{K}_{\mathbf{X}\mathbf{Z}}\mathbf{K}_{\mathbf{Z}\mathbf{Z}}^{-1}\mathbf{K}_{\mathbf{Z}\mathbf{X}})$
where, with a slight abuse of notation, $\mathbf{u}$ is treated as a column vector here, $\mathbf{K}_{\mathbf{X}\mathbf{Z}}\triangleq \mathbf{K}^{\top}_{\mathbf{Z}\mathbf{X}}$, and $\mathbf{K}_{\mathbf{Z}\mathbf{Z}}$ and $\mathbf{K}_{\mathbf{Z}\mathbf{X}}$ denote covariance matrices with components $k(\mathbf{z}_m, \mathbf{z}_{m'})$ for $m,m'=1,\ldots,M$ and $k(\mathbf{z}_m, \mathbf{x}_{n})$ for $m =1,\ldots,M$ and $n=1,\ldots,N$, respectively.
The SGP predictive belief can also be computed in closed form by marginalizing out $\mathbf{u}$: 
$p(\mathbf{f}^\star|\mathbf{y}) = \int p(\mathbf{f}^\star|\mathbf{u})\ p(\mathbf{u}|\mathbf{y})\ \mathrm{d}\mathbf{u}$.
        
The work of~\cite{Titsias09} has proposed a principled \emph{variational inference} (VI) framework that  approximates the joint posterior belief $p(\mathbf{f}, \mathbf{u}|\mathbf{y})$ with a variational posterior $q(\mathbf{f}, \mathbf{u}) \triangleq p(\mathbf{f}|\mathbf{u})\ q(\mathbf{u})$ by minimizing the \emph{Kullback-Leibler} (KL) distance between them, 
which is equivalent to maximizing a lower bound of the log-marginal likelihood (i.e., also known as the \emph{evidence lower bound} (ELBO)):
        \begin{equation*}
            \mathrm{ELBO} \triangleq \mathbb{E}_{q(\mathbf{f})}[\log p(\mathbf{y}|\mathbf{f})] - \mathrm{KL}[q(\mathbf{u})\Vert p(\mathbf{u})]
        \end{equation*}
where $q(\mathbf{f}) \triangleq \int p(\mathbf{f}|\mathbf{u})\ q(\mathbf{u})\ \mathrm{d}\mathbf{u}$. 
        A common choice in VI is the Gaussian variational posterior $q(\mathbf{u}) \triangleq \mathcal{N}(\mathbf{u}|\mathbf{m}, \mathbf{S})$ of the inducing variables $\mathbf{u}$~\cite{deisenroth2015distributed,Yarin14,Lawrence13,NghiaICML16,HoangICML16,Titsias09a} which results in a Gaussian marginal $q(\mathbf{f})=\mathcal{N}(\mathbf{f}|\boldsymbol{\mu}, \mathbf{\Sigma})$ where $\boldsymbol{\mu} \triangleq \mathbf{K}_{\mathbf{X}\mathbf{Z}}\mathbf{K}_{\mathbf{Z}\mathbf{Z}}^{-1}\mathbf{m}$ and $\mathbf{\Sigma} \triangleq \mathbf{K}_{\mathbf{X}\mathbf{X}} - \mathbf{K}_{\mathbf{X}\mathbf{Z}}\mathbf{K}_{\mathbf{Z}\mathbf{Z}}^{-1}(\mathbf{K}_{\mathbf{Z}\mathbf{Z}} - \mathbf{S})\mathbf{K}_{\mathbf{Z}\mathbf{Z}}^{-1}\mathbf{K}_{\mathbf{Z}\mathbf{X}}$. 

\textbf{Deep Gaussian Process (DGP).} 
    \label{subsec: deep gaussian process}
%
        A multi-layer DGP model is a hierarchical composition of GP models. Consider a DGP with a depth of $L$ such that each DGP layer is associated with a set $\mathbf{F}_{\ell-1}$ of inputs and a set $\mathbf{F}_{\ell}$ of outputs for $\ell = 1, \dots, L$ and $\mathbf{F}_0 \triangleq \mathbf{X}$. Let $\boldsymbol{\mathcal{F}}\triangleq\{\mathbf{F}_{\ell}\}_{\ell=1}^L$, and the inducing inputs and corresponding inducing output variables for DGP layers $\ell = 1, \dots, L$ be denoted by the respective sets $\boldsymbol{\mathcal{Z}}\triangleq\{\mathbf{Z}_{\ell}\}_{\ell=1}^L$ and $\boldsymbol{\mathcal{U}}\triangleq\{\mathbf{U}_{\ell}\}_{\ell=1}^L$. Similar to the joint probability distribution of the SGP model~\eqref{crap},
        \begin{equation*}
            p(\mathbf{y}, \boldsymbol{\mathcal{F}}, \boldsymbol{\mathcal{U}})
            =  \underbrace{p(\mathbf{y}|\mathbf{F}_L)}_\text{data likelihood}\;
            \underbrace{\left[\prod_{\ell=1}^{L}p(\mathbf{F}_{\ell}|\mathbf{U}_{\ell})\right]p(\boldsymbol{\mathcal{U}})}_\text{DGP prior}.
        \end{equation*}
        Similarly, the variational posterior is assumed to be $q(\boldsymbol{\mathcal{F}}, \boldsymbol{\mathcal{U}}) \triangleq \left[\prod_{\ell=1}^{L}p(\mathbf{F}_{\ell}|\mathbf{U}_{\ell})\right]q(\boldsymbol{\mathcal{U}})$, thus resulting in the following ELBO for the DGP model:
        \begin{equation}
            \mathrm{ELBO} \triangleq \int q(\mathbf{F}_L) \log p(\mathbf{y}|\mathbf{F}_L)\ \mathrm{d}\mathbf{F}_L  - \mathrm{KL}[q(\boldsymbol{\mathcal{U}})\Vert p(\boldsymbol{\mathcal{U}})] 
            \label{eq: elboDGP}
        \end{equation} 
        where $q(\mathbf{F}_L) \triangleq \int \prod_{\ell=1}^{L}p(\mathbf{F}_{\ell}|\mathbf{U}_{\ell},\mathbf{F}_{\ell-1})\ q(\boldsymbol{\mathcal{U}})\ \mathrm{d}\mathbf{F}_1\dots \mathrm{d}\mathbf{F}_{L-1}\mathrm{d}\boldsymbol{\mathcal{U}}$. To compute $q(\mathbf{F}_L)$, the work of~\cite{salimbeni2017doubly} has proposed the use of the reparameterization trick~\cite{kingma2013auto} and Monte Carlo sampling, which are adopted in this work.
%
%
\begin{remark}
To the best of our knowledge, the DGP models exploiting the inducing variables\footnote{An alternative is to modify the DGP prior directly and perform inference with a parametric model. The work of~\cite{cutajar2017random} has approximated the DGP prior with the spectral density of a kernel~\cite{NghiaAAAI17} such that the kernel has an analytical spectral density.} and the VI framework \cite{dai2015variational,damianou2013deep,hensman2014nested,salimbeni2017doubly} have imposed the highly restrictive assumptions of (i) mean field approximation $q(\boldsymbol{\mathcal{U}}) \triangleq \prod_{\ell=1}^{L}q(\mathbf{U}_{\ell})$ and (ii) biased Gaussian variational posterior $q(\mathbf{U}_{\ell})$. 
        In fact, the true posterior belief usually exhibits a high correlation across the DGP layers and is non-Gaussian~\cite{havasi2018inference}, hence potentially compromising the performance of such deterministic approximation methods for DGP models.
To remove these assumptions, we will propose a principled approximation method that can generate unbiased posterior samples even under the VI framework, as detailed in Section~\ref{sec: IPVI DGP}.
\label{rem1}
\end{remark}
%
\section{Implicit Posterior Variational Inference (IPVI) for DGPs} 
\label{sec: IPVI DGP}
%
Unlike the conventional VI framework for existing DGP models~\cite{dai2015variational,damianou2013deep,hensman2014nested,salimbeni2017doubly}, our proposed IPVI framework does not need to impose their highly restrictive assumptions (Remark~\ref{rem1}) and can still preserve the time efficiency of VI.
Inspired by previous works on adversarial-based inference~\cite{huszar2017variational,mescheder2017adversarial},  
IPVI achieves this by first generating posterior samples $\boldsymbol{\mathcal{U}}\triangleq g_\Phi(\epsilon)$ with a black-box \textbf{generator} $g_\Phi(\epsilon)$ parameterized by $\Phi$ and a  random noise $\epsilon\sim \mathcal{N}(\mathbf{0},\mathbf{I})$. By representing the variational posterior as $q_{\Phi}(\boldsymbol{\mathcal{U}})\triangleq\int p(\boldsymbol{\mathcal{U}}|\epsilon)\mathrm{d}\epsilon$, the ELBO in~\eqref{eq: elboDGP} can be re-written as
        \begin{equation}
            \mathrm{ELBO} = \mathbb{E}_{q(\mathbf{F}_L)}[\log p(\mathbf{y}|\mathbf{F}_L)] - \mathrm{KL}[q_\Phi(\boldsymbol{\mathcal{U}})\Vert p(\boldsymbol{\mathcal{U}})]\ .
            \label{eq: phi elbo}
        \end{equation}
    An immediate advantage of the generator $g_\Phi(\epsilon)$ is that it can generate the posterior samples in parallel by feeding it a batch of randomly sampled $\epsilon$'s.
However, representing the variational posterior $q_{\Phi}(\boldsymbol{\boldsymbol{\mathcal{U}}})$ implicitly makes it impossible to evaluate the KL distance in~\eqref{eq: phi elbo} since $q_{\Phi}(\boldsymbol{\boldsymbol{\mathcal{U}}})$ cannot be calculated explicitly. 
By observing that the KL distance is equal to the expectation of the log-density ratio $\mathbb{E}_{q_\Phi(\boldsymbol{\mathcal{U}})}[\log q_\Phi(\boldsymbol{\boldsymbol{\mathcal{U}}}) - \log p(\boldsymbol{\boldsymbol{\mathcal{U}}})]$, we can circumvent an explicit calculation of the KL distance term by implicitly representing the log-density ratio as a separate function $T$ to be optimized, as shown in our first result below:
%
    \begin{prop}
                \label{prop: optimal T(u)}
            Let $\sigma(x)\triangleq{1}/{(1+\exp(-x))}$. Consider the following maximization problem:
                \begin{equation}
                  \max_{T}\  \mathbb{E}_{p(\boldsymbol{\mathcal{U}})}[\log (1-\sigma(T(\boldsymbol{\mathcal{U}}))] + \mathbb{E}_{q_{\Phi}(\boldsymbol{\mathcal{U}})}[\log \sigma(T(\boldsymbol{\mathcal{U}}))]\ .
                \label{eq: discriminator}
                \end{equation}
If $p(\boldsymbol{\mathcal{U}})$ and $q_\Phi(\boldsymbol{\mathcal{U}})$ are known, then the optimal $T^*$ with respect to~\eqref{eq: discriminator} is the log-density ratio:
            \begin{equation}
            T^*(\boldsymbol{\mathcal{U}}) = \log q_\Phi(\boldsymbol{\mathcal{U}}) - \log p(\boldsymbol{\mathcal{U}})\ .
            \label{voodoo}
            \end{equation}
    \end{prop}
Its proof 
(Appendix~\ref{append: proof of Tpsiu}) is similar to that of Proposition $1$ in~\cite{goodfellow2014generative} except that we use a sigmoid function $\sigma$ to reveal the log-density ratio. 
Note that~\eqref{eq: discriminator} defines a binary cross-entropy between samples from the variational posterior $q_\Phi(\boldsymbol{\mathcal{U}})$ and prior $p(\boldsymbol{\mathcal{U}})$. Intuitively, $T$ in~\eqref{eq: discriminator}, which we refer to as a \textbf{discriminator}, tries to distinguish between $q_\Phi(\boldsymbol{\mathcal{U}})$ and $p(\boldsymbol{\mathcal{U}})$ by outputting $\sigma(T(\boldsymbol{\mathcal{U}}))$ as the probability of $\boldsymbol{\mathcal{U}}$ being a sample from $q_\Phi(\boldsymbol{\mathcal{U}})$ rather than $p(\boldsymbol{\mathcal{U}})$. 

Using Proposition~\ref{prop: optimal T(u)} (i.e.,~\eqref{voodoo}), the ELBO in~\eqref{eq: phi elbo} can be re-written as
        \begin{equation}
            \mathrm{ELBO} = \mathbb{E}_{q_{\Phi}(\boldsymbol{\mathcal{U}})}[\mathcal{L}(\theta, \mathbf{X}, \mathbf{y}, \boldsymbol{\mathcal{U}}) - T^*(\boldsymbol{\mathcal{U}}) ]
        \label{eq: new elbo}
        \end{equation}
where $\mathcal{L}(\theta, \mathbf{X}, \mathbf{y}, \boldsymbol{\mathcal{U}}) \triangleq \mathbb{E}_{p(\mathbf{F}_L|\boldsymbol{\mathcal{U}})}[\log p(\mathbf{y}|\mathbf{F}_L)]$ and $\theta$ denotes the  DGP model hyperparameters.        
The ELBO can now be calculated given the optimal discriminator $T^*$. In our implementation, we adopt a parametric representation for discriminator $T$. In principle, the parametric representation is required to be expressive enough to be able to represent the optimal discriminator $T^*$ accurately.
        Motivated by the fact that deep neural networks are universal function approximators \cite{hornik1989multilayer}, we represent discriminator $T_{\Psi}$ by a neural network with parameters $\Psi$; the optimal $T_{\Psi^*}$ is thus parameterized by $\Psi^*$. The architecture of the generator and discriminator in our IPVI framework will be discussed in Section~\ref{sec: arichitecture}.

The ELBO in~\eqref{eq: new elbo} can be optimized with respect to $\Phi$ and $\theta$ via gradient ascent, provided that the optimal $T_{\Psi^*}$ (with respect to $q_\Phi$) can be obtained in every iteration. 
One way to achieve this
%
%
 is to cast the optimization of the ELBO as a two-player pure-strategy
         game between \textbf{Player 1} (representing discriminator with strategy $\{\Psi\}$) vs. \textbf{Player 2} (jointly representing generator and DGP model with strategy $\{\Phi, \theta\}$) that is defined based on the following payoffs:
        \begin{equation}
        \begin{array}{rl}
            \displaystyle \textbf{Player 1 :} & \displaystyle\max_{\{\Psi\}}\ \mathbb{E}_{p(\boldsymbol{\mathcal{U}})}[\log (1-\sigma(T_{\Psi}(\boldsymbol{\mathcal{U}}))] + \mathbb{E}_{q_{\Phi}(\boldsymbol{\mathcal{U}})}[\log \sigma(T_{\Psi}(\boldsymbol{\mathcal{U}}))] \ ,\\
            \displaystyle \textbf{Player 2 :} & \displaystyle\max_{\{\theta, \Phi\}}\ \mathbb{E}_{q_{\Phi}(\boldsymbol{\mathcal{U}})}[\mathcal{L}(\theta, \mathbf{X}, \mathbf{y}, \boldsymbol{\mathcal{U}}) - T_{\Psi}(\boldsymbol{\mathcal{U}}) ]\ .
        \end{array}
        \label{eq: payoff}
        \end{equation}
        \begin{prop}
            Suppose that the parametric representations of $T_\Psi$ and $g_\Phi$ are expressive enough to represent any function. If $(\{\Psi^*\}, \{\theta^*, \Phi^*\})$ is a Nash equilibrium of the game in~\eqref{eq: payoff}, then
            $\{\theta^*, \Phi^*\}$ is a global maximizer of the ELBO in~\eqref{eq: phi elbo} such that
            (a) $\theta^*$ is the maximum likelihood assignment for the DGP model, and 
            (b) $q_{\Phi^*}(\boldsymbol{\mathcal{U}})$ is equal to the true posterior belief $p(\boldsymbol{\mathcal{U}}|\mathbf{y})$.            
        \label{prop: nash equilibrium}
        \end{prop}
Its proof is similar to that of Proposition $3$ in~\cite{mescheder2017adversarial} except that we
additionally provide a proof of existence of a Nash equilibrium for the case of known/fixed DGP model hyperparameters, as detailed in Appendix~\ref{append: proof of nash equilibrium}. Proposition~\ref{prop: nash equilibrium} reveals that any Nash equilibrium coincides with a global maximizer of the original ELBO  in~\eqref{eq: phi elbo}. 
This consequently inspires us to play the game using \emph{best-response dynamics}\footnote{This procedure is sometimes called ``better-response dynamics'' (\url{http://timroughgarden.org/f13/l/l16.pdf}).} (BRD) which is a commonly adopted procedure~\cite{awerbuch2008fast} to search for a Nash equilibrium.
%
Fig.~\ref{fig: game} illustrates our BRD algorithm: 
In each iteration of Algorithm~\ref{alg: Main1}, each player takes its turn to improve its strategy to achieve a better (but not necessarily the best) payoff by performing a \emph{stochastic gradient ascent} (SGA) update on its payoff~\eqref{eq: payoff}.
%
        
        \begin{figure}   
            \begin{tabular}{ll}
                \begin{minipage}{0.40\linewidth}    
                \begin{algorithm}[H]
                \label{alg: Main1}                
                Randomly initialize $\theta$, $\Psi$, $\Phi$\\
                \While{not converged}
                {
                   Run Algorithm~\ref{alg: D}\\
                   Run Algorithm~\ref{alg: G&GP}
                }
                \caption{Main}
                \end{algorithm}
                \end{minipage}
                &
                \begin{minipage}{0.56\linewidth}
                
                    \centering
                    \scalebox{0.6}{\hspace{-28mm}
                    \begin{tikzpicture}
                        \node[circle,fill=yellow!100] (0) at (-10cm,0cm) {\begin{tabular}{c}$\{\Psi^{0}\}$\vspace{1mm} \\ $\{\theta^{0},\Phi^{0}\}$ \end{tabular}};
                        \node[circle,fill=yellow!100] (f1) at (-7cm,0cm) {\begin{tabular}{c}$\{\Psi^{1}\}$\vspace{1mm} \\ $\{\theta^{0},\Phi^{0}\}$ \end{tabular}};
                        \path[->] (0) edge (f1);
                        \node[text width=10cm, text centered] at (-8.5cm, .8cm) {\textbf{Player 1}};

                        \node[text width=10cm, text centered] at (-5.5cm, .8cm) {\textbf{Player 2}}; 
                        \node[circle,fill=yellow!100] (tp1) at (-4cm,0cm) {\begin{tabular}{c}$\{\Psi^{1}\}$\vspace{1mm} \\ $\{\theta^{1},\Phi^{1}\}$ \end{tabular}};
            
                        \node[circle,fill=yellow!100] (f2) at (-1cm,-0cm) {\begin{tabular}{c}$\{\Psi^{2}\}$\vspace{1mm} \\ $\{\theta^{1},\Phi^{1}\}$ \end{tabular}};
                        \path[->] (tp1) edge (f2);
                        \path[->] (f1) edge (tp1);
                        \node[text width=10cm, text centered] at (-2.5cm, .8cm) {\textbf{Player 1}}; 
                        \node[circle] (y) at (2cm,-0cm) {$\dots$};
                        \path[->] (f2) edge (y);
                        \node[text width=10cm, text centered] at (.5cm, .8cm) {\textbf{Player 2}};
                    \end{tikzpicture}   }  
            \end{minipage}
        \end{tabular}
        \begin{tabular}{ll}
            \begin{minipage}{0.48\linewidth}    
            \begin{algorithm}[H]
                \label{alg: D}            
                Sample $\{\boldsymbol{\mathcal{V}}_1,\ldots,\boldsymbol{\mathcal{V}}_K\}$ from $p(\boldsymbol{\mathcal{U}})$ \\
                Sample $\{\boldsymbol{\mathcal{U}}_1,\ldots,\boldsymbol{\mathcal{U}}_K\}$ from $q_{\Phi}(\boldsymbol{\mathcal{U}})$ \\
                
                Compute gradient w.r.t. $\Psi$ from~\eqref{eq: payoff}: \\
    $
                \begin{array}{rl}
                g_{\Psi} \triangleq \hspace{-2.4mm}& \displaystyle\nabla_{\Psi}\hspace{-1mm} \left[ \frac{1}{K}\sum_{k=1}^{K}\log (1-\sigma(T_{\Psi}(\boldsymbol{\mathcal{V}}_k))\right] \\
                &\displaystyle+\nabla_{\Psi}\hspace{-1mm} \left[ \frac{1}{K}\sum_{k=1}^{K}\log \sigma(T_{\Psi}(\boldsymbol{\mathcal{U}}_k))\right]
                \end{array}
                $\\
                SGA update for $\Psi$: \\
                $
                \begin{array}{l}
                \Psi \leftarrow \Psi + \alpha_{\Psi}\ g_{\Psi}
                \end{array}
                $ 
                
                
                \caption{\textbf{Player 1}}
            \end{algorithm}
            \end{minipage}
                &
                \begin{minipage}{0.46\linewidth}
                \begin{algorithm}[H]
                    \label{alg: G&GP}                
                        Sample mini-batch $(\mathbf{X}_b, \mathbf{y}_b)$ from $(\mathbf{X}, \mathbf{y})$ \\                        
                        Sample $\{\boldsymbol{\mathcal{U}}_1,\ldots,\boldsymbol{\mathcal{U}}_K\}$ from $q_{\Phi}(\boldsymbol{\mathcal{U}})$
        
                        Compute gradients w.r.t. $\theta$ and $\Phi$ from~\eqref{eq: payoff}:\\
                        $\hspace{-1.7mm}
                        \begin{array}{rl}
g_{\theta}  \triangleq \hspace{-2.4mm}& \displaystyle\nabla_{\theta}\hspace{-1mm}\left[\frac{1}{K}\hspace{-0.5mm}\sum_{k=1}^{K}  \mathcal{L}(\theta,\mathbf{X}_b, \mathbf{y}_b, \boldsymbol{\mathcal{U}}_k)\right]\\
                        g_{\Phi} \triangleq \hspace{-2.4mm}& \displaystyle\nabla_{\Phi}\hspace{-1mm}\left[\frac{1}{K}\hspace{-0.5mm}\sum_{k=1}^{K}  \mathcal{L}(\theta,\mathbf{X}_b, \mathbf{y}_b, \boldsymbol{\mathcal{U}}_k)\hspace{-0.5mm} -\hspace{-0.5mm} T_{\Psi}(\boldsymbol{\mathcal{U}}_k)\right] \\
                        \end{array}
                        $\\
                        SGA updates for $\theta$ and $\Phi$: \\
                        $
                        \begin{array}{l}
                        \theta \leftarrow \theta + \alpha_{\theta}\  g_{\theta}\ ,\quad \Phi \leftarrow \Phi + \alpha_{\Phi}\ g_{\Phi}
                        \end{array}
                        $                        
                    \caption{\textbf{Player 2}}
                \end{algorithm}
                \end{minipage}
        \end{tabular}
        \caption{\emph{Best-response dynamics} (BRD) algorithm based on our IPVI framework for DGPs.}
        \label{fig: game}
        \end{figure}
%
\begin{remark}
While BRD guarantees to converge to a Nash equilibrium in some classes of games (e.g., a finite potential game), we have not shown that our game falls into any of these classes and hence cannot guarantee that BRD converges to a Nash equilibrium (i.e., global maximizer $\{\theta^*, \Phi^*\}$) of our game.
Nevertheless, as mentioned previously, obtaining the optimal discriminator in every iteration guarantees the game play (i.e., gradient ascent update for $\{\theta,\Phi\}$) to reach at least a local maximum of ELBO. To better approximate the optimal discriminator, we perform multiple calls of Algorithm~\ref{alg: D}             in every iteration of the main loop in Algorithm~\ref{alg: Main1} and also apply a larger learning rate $\alpha_\Psi$.
We have observed in our own experiments that these tricks improve the predictive performance of IPVI. 
\end{remark}

\begin{remark}
    \textcolor{red}{Existing implicit VI frameworks~\cite{titsias2019unbiased,yin2018semi} avoid the estimation of the log-density ratio. Unfortunately, the semi-implicit VI framework of  \cite{yin2018semi} requires taking a limit at infinity to recover the ELBO, while the unbiased implicit VI framework of~\cite{titsias2019unbiased} relies on a Markov chain Monte Carlo sampler whose hyperparameters need to be carefully tuned.}  
\end{remark}

%
\section{Parameter-Tying Architecture of Generator and Discriminator for DGPs} 
\label{sec: arichitecture}
\begin{figure}
      \begin{tabular}{llll}
          \hspace{-2mm}\includegraphics[height=27mm]{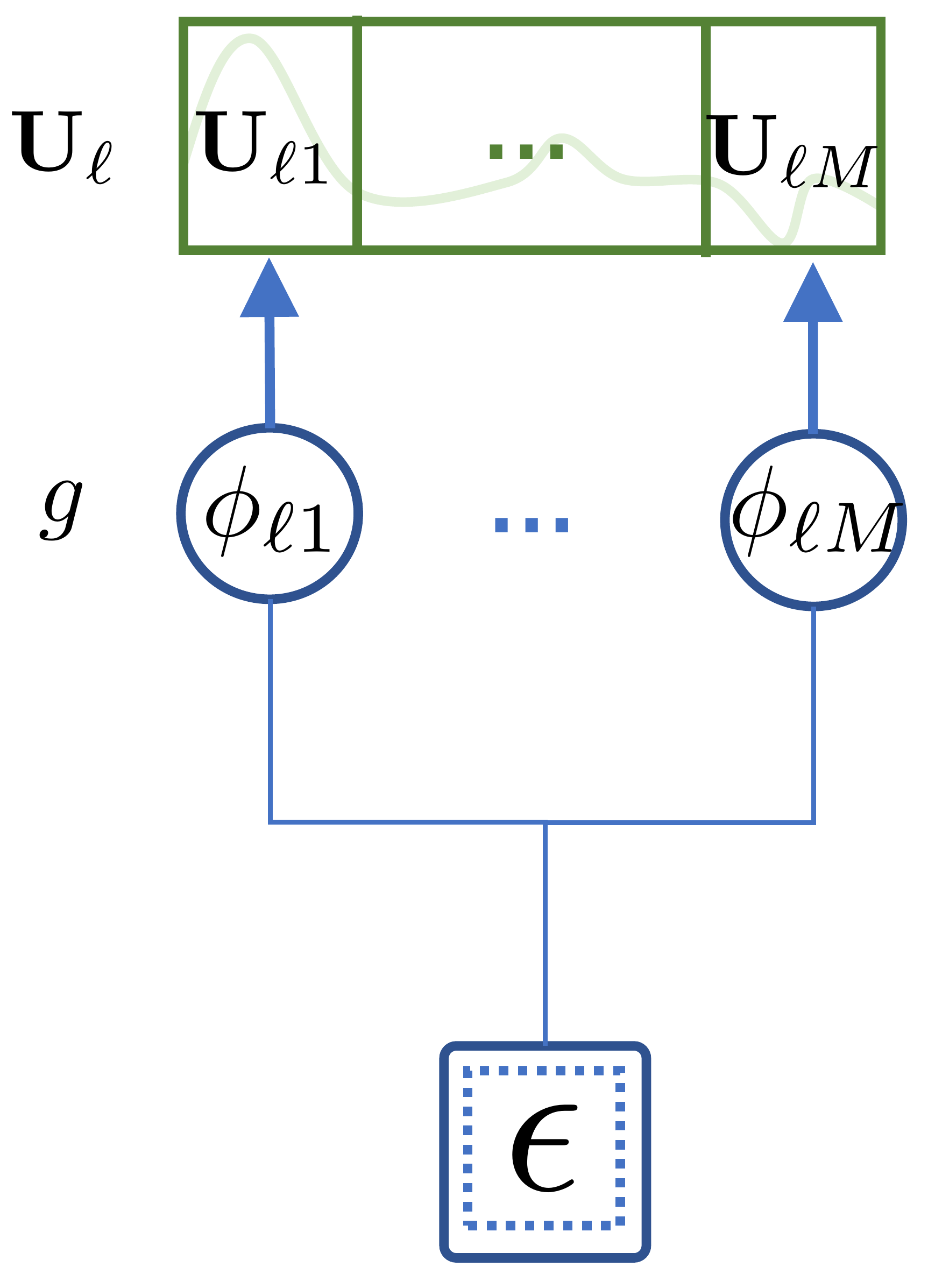} 
          & 
          \hspace{-2mm}\includegraphics[height=27mm]{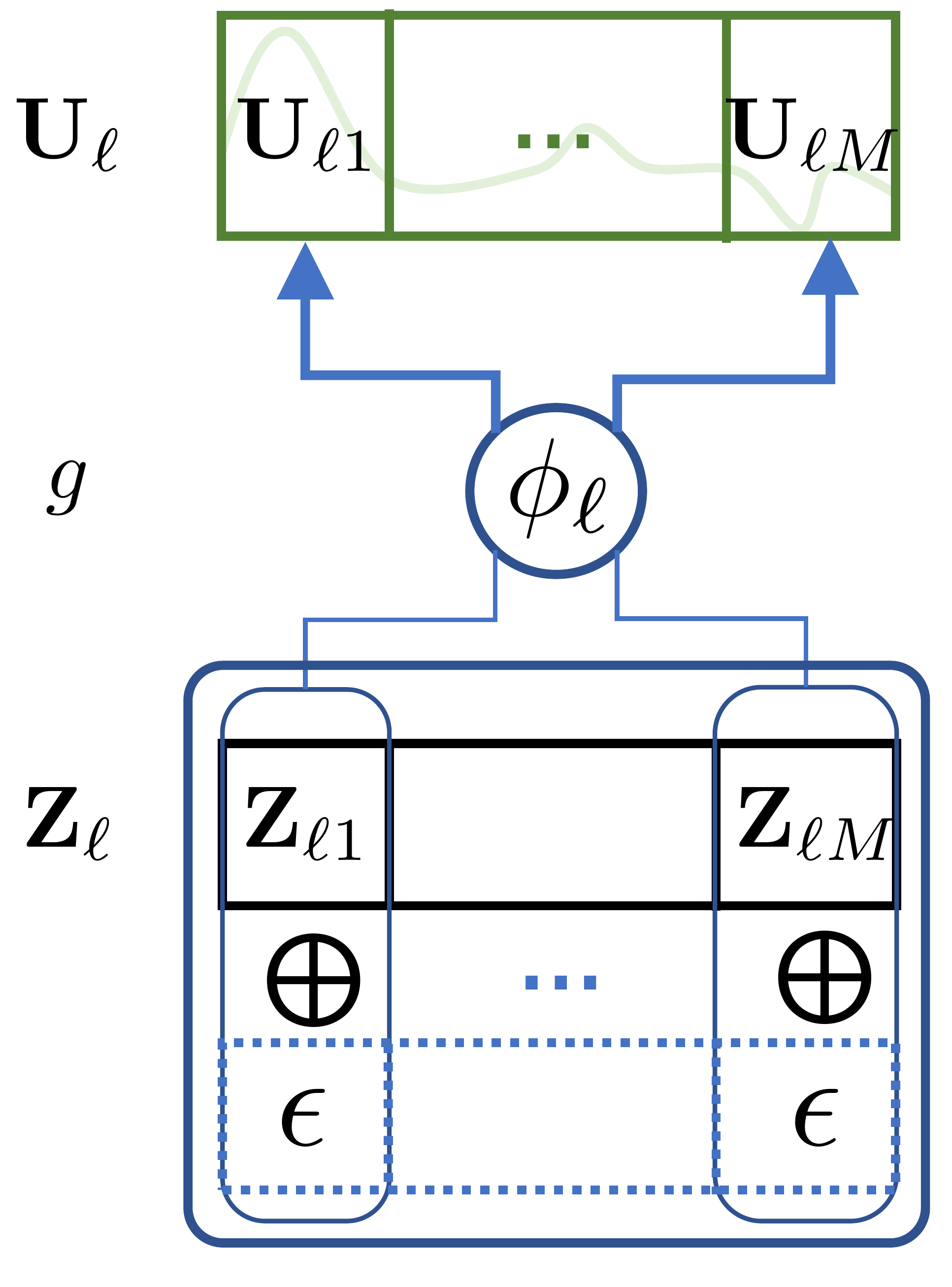} 
          & 
          \hspace{-1mm}\includegraphics[height=27mm]{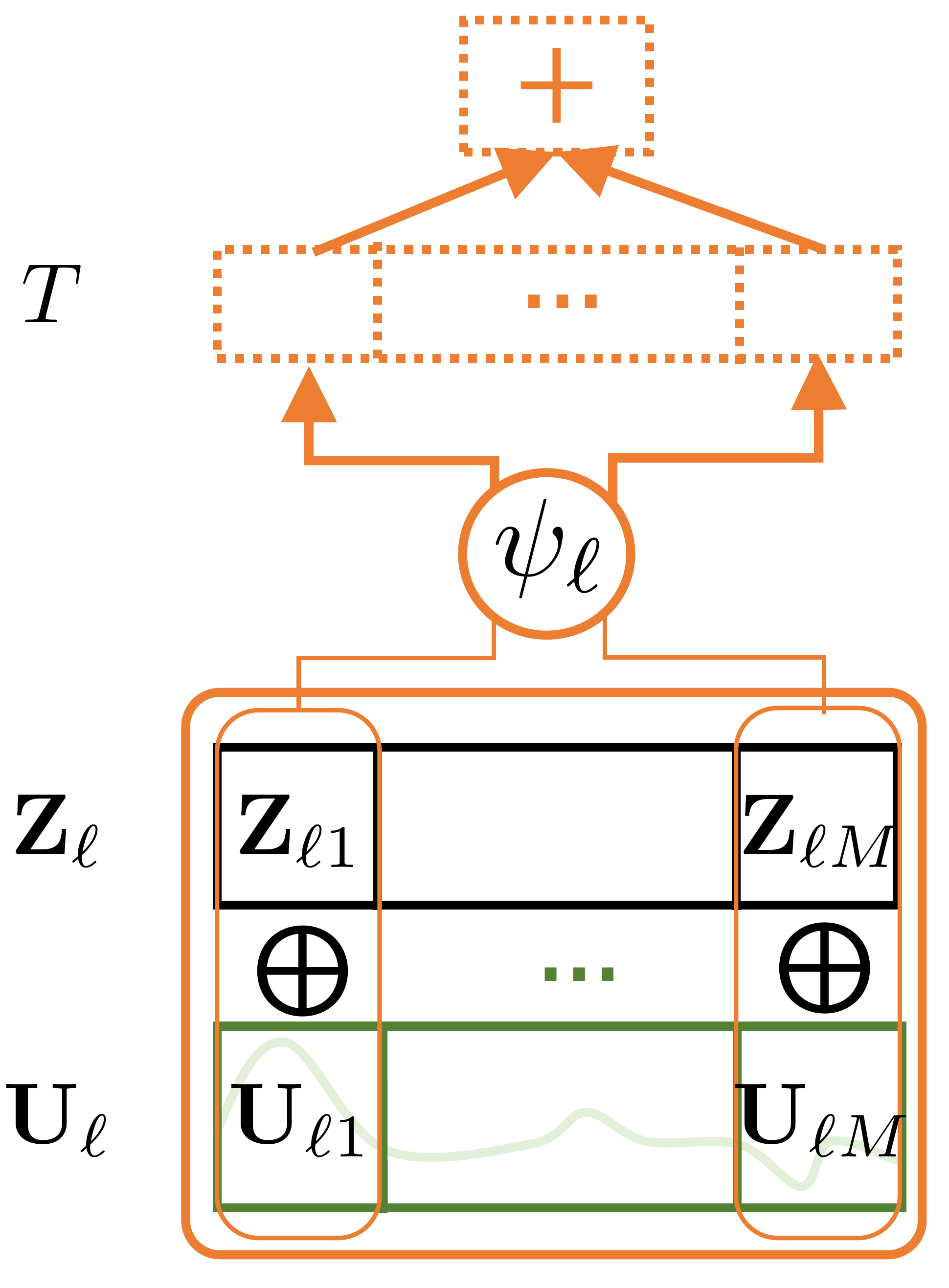}
          &
          \hspace{-0mm}\includegraphics[height=27mm]{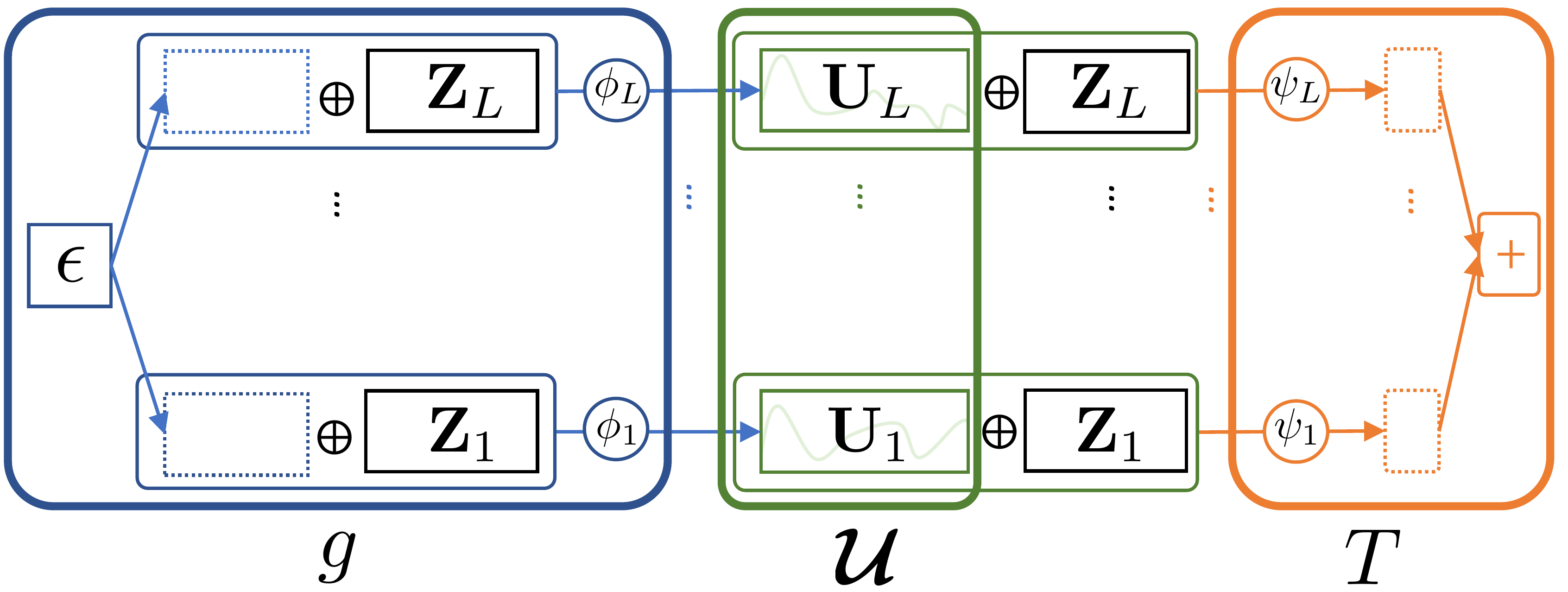}
          \\
          \hspace{7.5mm}(a) & \hspace{7.6mm}(b) & \hspace{7.4mm} (c) & \hspace{37.5mm}(d)
      \end{tabular}
      \caption{ (a) Illustration of a naive design of the generator for each layer $\ell$.
      Parameter-tying architecture of the (b) generator and (c) discriminator for each layer $\ell$ where 
       `$+$' denotes addition and `$\oplus$' denotes concatenation.
       (d) Parameter-tying architecture of the generator and discriminator in our IPVI framework for DGPs. 
See the main text for the definitions of notations.}
      \label{fig:architecture}
    \end{figure}
In this section, we will discuss how the architecture of the generator 
in our IPVI framework is designed to enable parameter tying for a DGP model to alleviate overfitting. 
Recall from Section~\ref{subsec: deep gaussian process} that $\boldsymbol{\mathcal{U}} = \{\mathbf{U}_{\ell}\}_{\ell=1}^{L}$ is a collection of inducing variables for DGP layers $\ell=1,\ldots,L$. 
We consider a layer-wise design of the generator (parameterized by $\Phi\triangleq\{\phi_{\ell}\}_{\ell=1}^{L}$) and discriminator (parameterized by $\Psi\triangleq\{\psi_{\ell}\}_{\ell=1}^{L}$)
such that $g_{\Phi}(\epsilon) \triangleq \{g_{\phi_{\ell}}(\epsilon)\}_{\ell=1}^{L}$ with the random noise $\epsilon$ serving as a common input to induce dependency between layers
and $T_{\Psi}(\boldsymbol{\mathcal{U}}) \triangleq \sum_{\ell=1}^{L}T_{\psi_{\ell}}(\mathbf{U}_{\ell})$, respectively. 
For each layer $\ell$, a naive design is to generate posterior samples $\mathbf{U}_{\ell}\triangleq g_{\phi_{\ell}}(\epsilon)$ from the random noise $\epsilon$ as input.
However, such a design suffers from two critical issues:
\textcolor{red}{
\squishlisttwo
    \item[$\bullet$] Fig.~\ref{fig:architecture}a illustrates that to generate posterior samples of $M$ different inducing variables $\mathbf{U}_{\ell 1},\ldots,\mathbf{U}_{\ell M}$  ($\mathbf{U}_{\ell}\triangleq\{\mathbf{U}_{\ell m} \}^{M}_{m=1}$), it is natural for the generator to adopt $M$ different parametric settings $\phi_{\ell 1},\ldots, \phi_{\ell M}$ ($\phi_{\ell}\triangleq\{ \phi_{\ell m}\}^M_{m=1}$), which introduces a relatively large number of parameters and is thus prone to overfitting (Section~\ref{subsec: tying vs no tying}).
    \item[$\bullet$] Such a design of the generator fails to adequately capture the dependency of the inducing output variables $\mathbf{U}_{\ell}$ on the corresponding inducing inputs $\mathbf{Z}_{\ell}$, hence restricting its capability to output the posterior samples of $\boldsymbol{\mathcal{U}}$ accurately.
\squishend
}

To resolve the above issues, we propose a novel parameter-tying architecture of the generator and discriminator for a DGP model, as shown in Figs.~\ref{fig:architecture}b and~\ref{fig:architecture}c. 
For each layer $\ell$, since $\mathbf{U}_\ell$ depends on $\mathbf{Z}_\ell$, 
we design the generator $g_{\phi_\ell}$ to generate posterior samples $\mathbf{U}_\ell \triangleq g_{\phi_\ell}(\epsilon \oplus \mathbf{Z}_\ell)$ from not just $\epsilon$ but also $\mathbf{Z}_\ell$ as inputs. \textcolor{red}{Recall that the same $\epsilon$ is fed as an input to $g_{\phi_l}$ in each layer $\ell$, which can be observed from the left-hand side of Fig.~\ref{fig:architecture}d.} In addition, compared with the naive design in Fig.~\ref{fig:architecture}a, the posterior samples of $M$ different inducing variables $\mathbf{U}_{\ell 1},\ldots,\mathbf{U}_{\ell M}$ are generated based on only a single shared parameter setting (instead of $M$), 
which reduces the number of parameters by $\mathcal{O}(M)$ times (Fig.~\ref{fig:architecture}b). 
%
%
We adopt a similar design for the discriminator, as shown in Fig.~\ref{fig:architecture}c. Fig.~\ref{fig:architecture}d illustrates the design of the overall parameter-tying architecture of the generator and discriminator. 

We have observed in our own experiments that our proposed parameter-tying architecture not only speeds up the training and prediction, but also improves the predictive performance of IPVI considerably (Section~\ref{subsec: tying vs no tying}).
We will empirically evaluate our IPVI framework with this parameter-tying architecture in Section~\ref{sec: experiments}.
%
%
%
\section{Experiments and Discussion} 
\label{sec: experiments}
%
%
We empirically evaluate and compare the performance of our IPVI framework\footnote{\textcolor{red}{Our implementation is built on GPflow \cite{GPflow2017} which is an open-source GP framework based on TensorFlow \cite{tf2016}. It is publicly available at \url{https://github.com/HeroKillerEver/ipvi-dgp}.}
} 
against that of the state-of-the-art 
SGHMC~\cite{havasi2018inference} and \emph{doubly stochastic VI}\footnote{It is reported in~\cite{salimbeni2017doubly} that DSVI has outperformed the approximate expectation propagation method of~\cite{bui2016deep} for DGPs. Hence, we do not empirically compare with the latter~\cite{bui2016deep} here.} (DSVI)~\cite{salimbeni2017doubly} for DGPs based on their publicly available implementations using synthetic and real-world datasets in supervised (e.g., regression and classification) and unsupervised learning tasks. 
    \subsection{Synthetic Experiment: Learning a Multi-Modal Posterior Belief} 
    \label{subsec: synthetic multi-modal data}
        To demonstrate the capability of IPVI in learning a complex multi-modal posterior belief, we generate a synthetic ``diamond'' dataset and adopt a multi-modal mixture of Gaussian prior belief $p(\mathbf{f})$ (see Appendix~\ref{append: toy experiment} for its description) to yield a multi-modal posterior belief $p(\mathbf{f}|\mathbf{y})$ for a single-layer GP. Fig.~\ref{fig: synthetic data}a illustrates this dataset and ground-truth posterior belief.         
Specifically, we focus on the multi-modal posterior belief $p(f|\mathbf{y}; x=0)$ at input $x=0$ whose ground truth is shown in Fig.~\ref{fig: synthetic data}d.
        Fig.~\ref{fig: synthetic data}c shows that as the number of parameters in the generator increases, the expressive power of IPVI increases such that its variational posterior $q(f;x=0)$ can capture more modes in the true posterior, thus resulting in a closer estimated \emph{Jensen-Shannon divergence} (JSD) between them and a higher \emph{mean log-likelihood} (MLL). 
        
Next, we compare the robustness of IPVI and SGHMC in learning the true multi-modal posterior belief $p(f|\mathbf{y};x=0)$ under different hyperparameter settings\footnote{We adopt scale-adapted SGHMC which is a robust variant used in Bayesian neural networks and DGP inference \cite{havasi2018inference}. A recent work of~\cite{zhang2019cyclical} has proposed the cyclical stochastic gradient MCMC method to improve the 
 accuracy of sampling highly complex distributions. However, it is not obvious to us how this method can be incorporated
into DGP models, which is beyond the scope of this work.}: 
The generators in IPVI use the same architecture with about $300$ parameters but 
different learning rates $\alpha_\Psi$, while the SGHMC samplers use different step sizes $\eta$. 
The results in Figs.~\ref{fig: synthetic data}b and~\ref{fig: synthetic data}e have verified a remark made in~\cite{zhang2019cyclical} that  
SGHMC is sensitive to the step size which cannot be set automatically~\cite{springenberg2016bayesian} and requires some prior knowledge to do so: 
Sampling with a small step size is prone to getting trapped in local modes while a slight increase of the step size may lead to an over-flattened posterior estimate.
Additional results for different hyperparameter settings of SGHMC can be found in Appendix~\ref{append: toy experiment}.
In contrast, the results in Figs.~\ref{fig: synthetic data}b and~\ref{fig: synthetic data}d reveal that, given enough parameters, IPVI performs robustly under a wide range of learning rates.
    \begin{figure}
         \begin{tabular}{lll}
            \hspace{5mm}
            \begin{tabular}{l}
                \hspace{-9mm} \includegraphics[height=32mm]{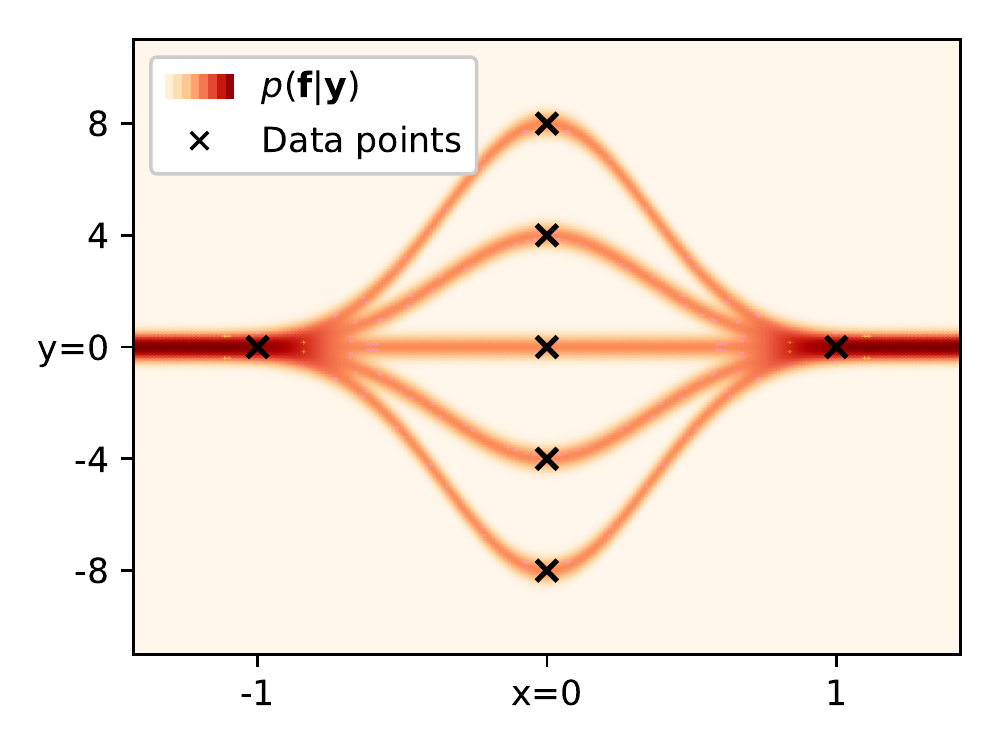}
                \vspace{-0mm}
                \\
                \vspace{-0mm}
                \hspace{13mm}(a) 
                \\
                \hspace{-8mm} \includegraphics[height=32mm]{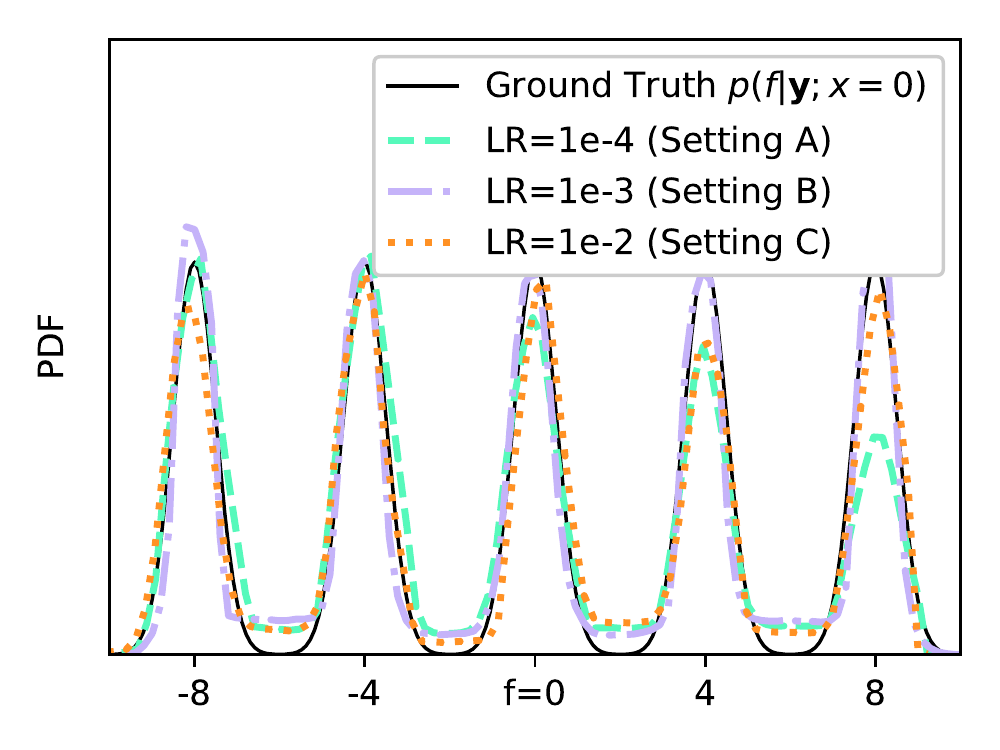}
                \vspace{-0mm}
                \\
                \hspace{13mm}(d)
            \end{tabular}
            &
            \hspace{-3mm}
            \vspace{0mm}
            \begin{tabular}{l}
            \hspace{-6mm}
                \hspace{0mm} \includegraphics[height=32mm]{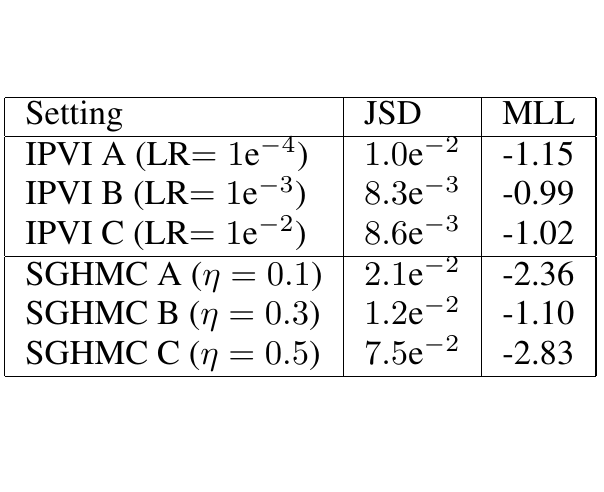}
            \vspace{-0mm}
            \\
            \hspace{13mm}(b)
            \\
            \hspace{-8mm} \includegraphics[height=32mm]{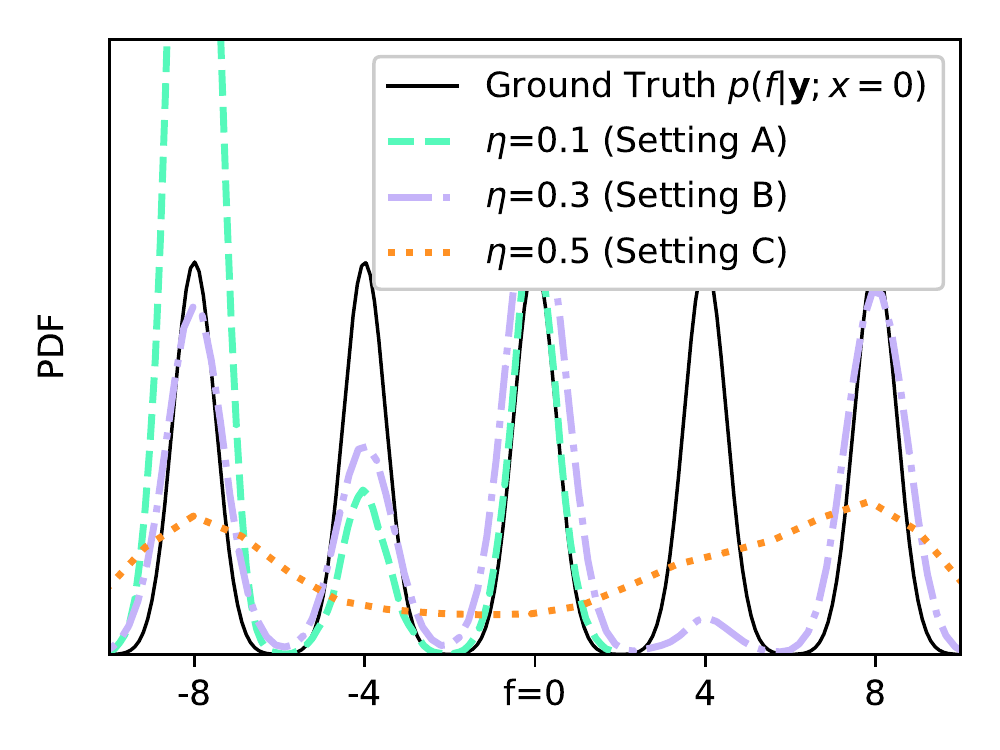}
            \\
            \hspace{13mm}(e)
            \end{tabular}
            &
            \hspace{-10.5mm} 
            \begin{tabular}{l}
            \includegraphics[height=70mm]{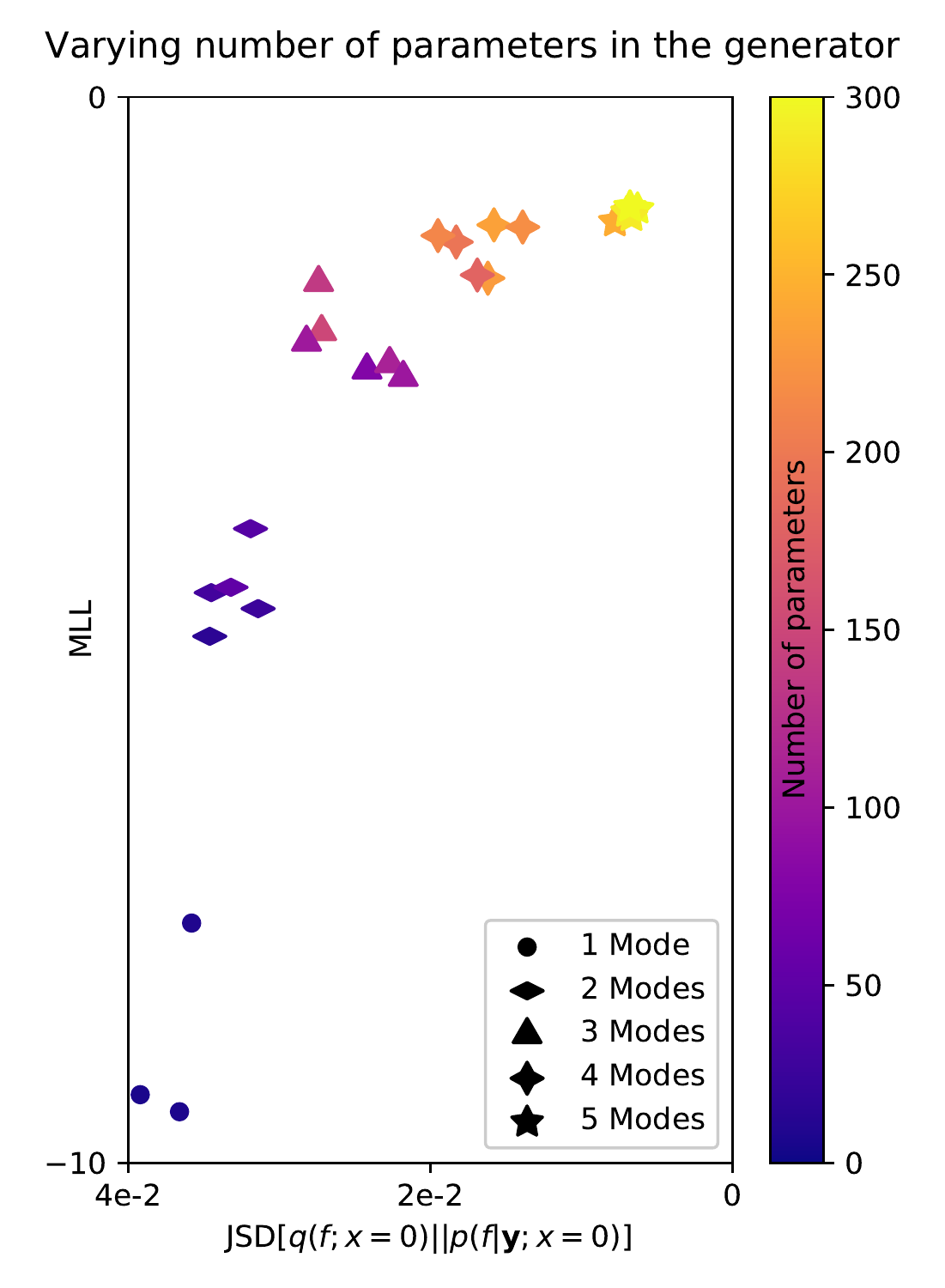}
            \vspace{-1mm}
            \\
            \hspace{22mm}(c)
            \end{tabular}
        \end{tabular}
        \caption{(a) The \emph{probability density function} (PDF) plot of the ground-truth posterior belief $p(\mathbf{f}|\mathbf{y})$. 
(b) Performances of IPVI and SGHMC in terms of estimated \emph{Jenson-Shannon divergence} (JSD) and \emph{mean log-likelihood} (MLL) metrics under the respective settings of varying learning rates $\alpha_\Psi$ and step sizes $\eta$.
(c) Graph of MLL vs. JSD achieved by IPVI with varying number of parameters in the generator: Different shapes indicate varying number of modes learned by the generator.
        (d-e) PDF plots of variational posterior $q(f;x=0)$ learned using (d) IPVI with generators of varying learning rates $\alpha_\Psi$ and (e) SGHMC with varying step sizes $\eta$.}
        \label{fig: synthetic data}
    \end{figure}    
\subsection{Supervised Learning: Regression and Classification} 
\label{subsec: uci benchmark regression}
For our experiments in the regression tasks, the depth $L$ of the DGP models are varied from $1$ to $5$ with $128$ inducing inputs per layer. The dimension of each hidden DGP layer is set to be (i) the same as the input dimension for the UCI benchmark regression and Airline datasets, (ii) $16$ for the YearMSD dataset, and (iii) $98$ for the classification tasks.
Additional details and results for our experiments (including that for IPVI with and without parameter tying) are found in Appendix~\ref{crappish}.

    
    \textbf{UCI Benchmark Regression.} Our experiments are first conducted on $7$ UCI benchmark regression datasets. We have performed a random $0.9/0.1$ train/test split.
        
    \textbf{Large-Scale Regression.} We then evaluate the performance of IPVI on two real-world large-scale regression datasets: (i) YearMSD dataset with a large input dimension $D=90$ and data size $N\approx 500000$, and (ii) Airline dataset with input dimension $D=8$ and a large data size $N\approx2$~million.
    For YearMSD dataset, we use the first $463715$ examples as training data and the last $51630$ examples as test data\footnote{This avoids the `producer' effect by ensuring that no song from an artist appears in both training \& test data.}.
    For Airline dataset, we set the last $100000$ examples as test data. 
    
    In the above regression tasks, the performance metric is the MLL of the test data (or test MLL). Fig.~\ref{fig:uci} shows results of the test MLL and standard deviation over $10$ runs. It can be observed that IPVI generally outperforms SGHMC and DSVI and the ranking summary shows that our IPVI framework for a $2$-layer DGP model (IPVI DGP 2) performs the best on average across all regression tasks. For large-scale regression tasks, the performance of IPVI consistently increases with a greater depth. Even for a small dataset, the performance of IPVI improves up to a certain depth.
 %
%
%
%
    \begin{figure}
    \hspace{2mm}
    \begin{tabular}{lllll}
        \hspace{-6mm} \includegraphics[height=50mm]{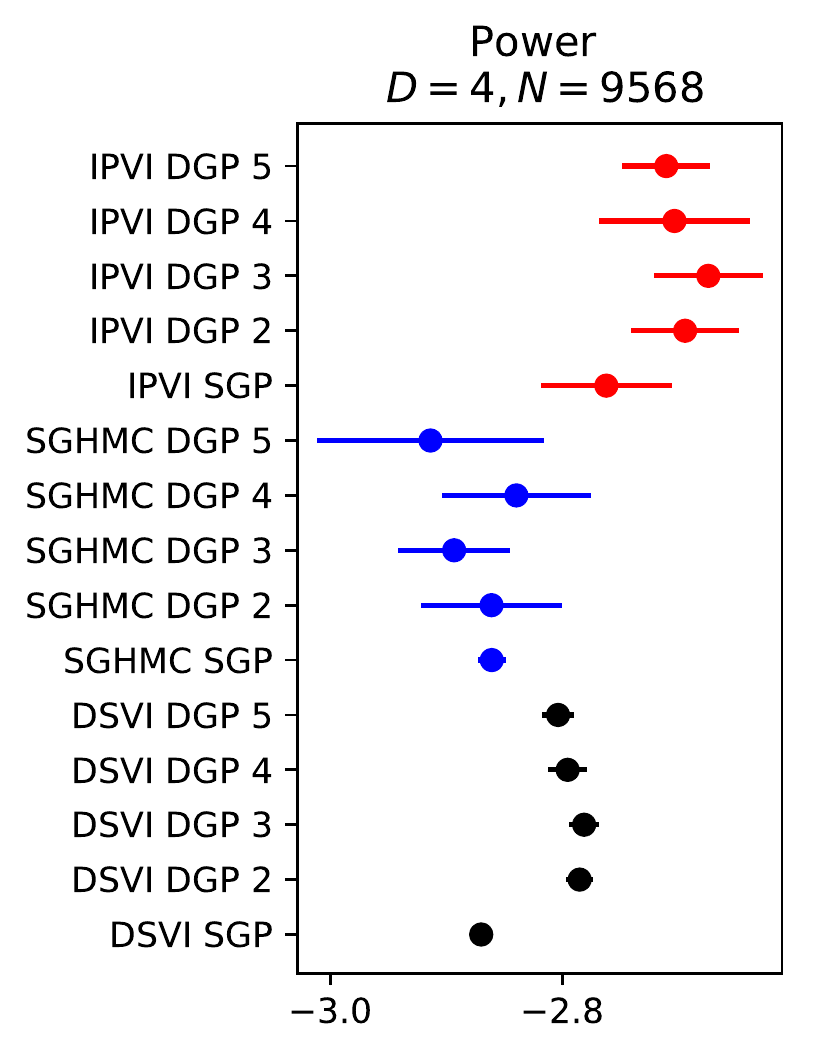}
        & \hspace{-6mm} \includegraphics[height=50mm]{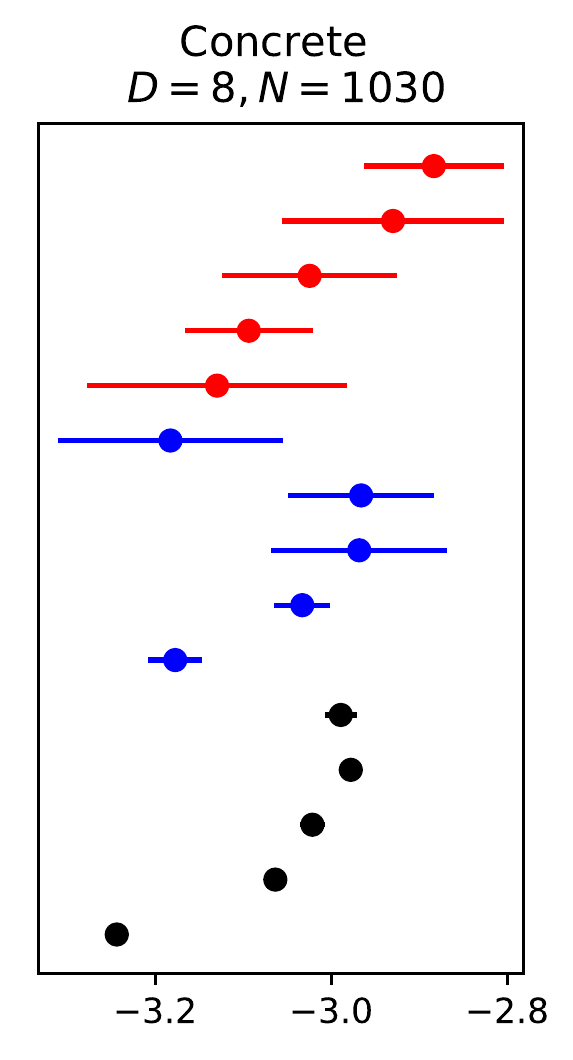}
        & \hspace{-6mm} \includegraphics[height=50mm]{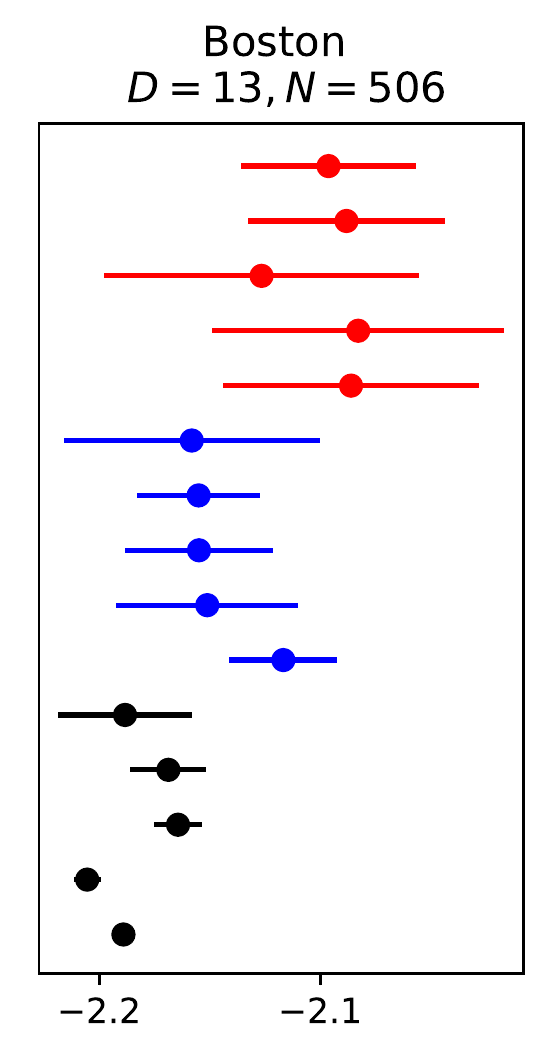}
        & \hspace{-6mm} \includegraphics[height=50mm]{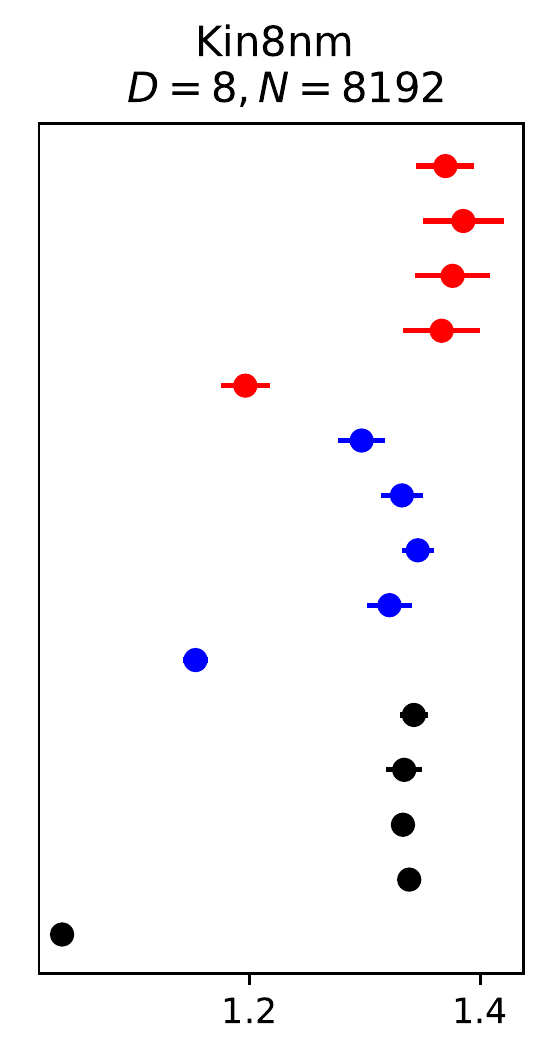}
        & \hspace{-6mm} \includegraphics[height=50mm]{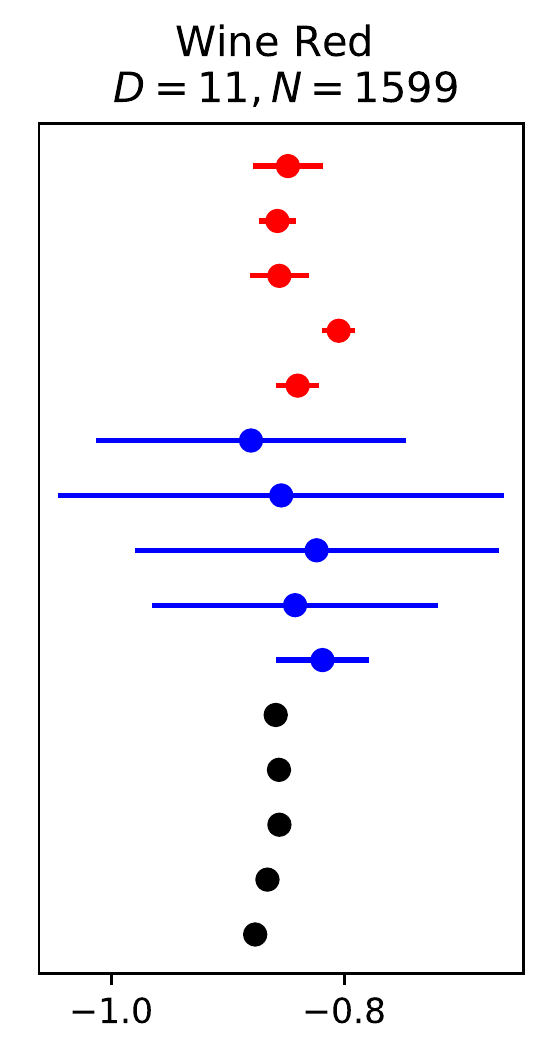}
        \\
         \hspace{-6mm} \includegraphics[height=50mm]{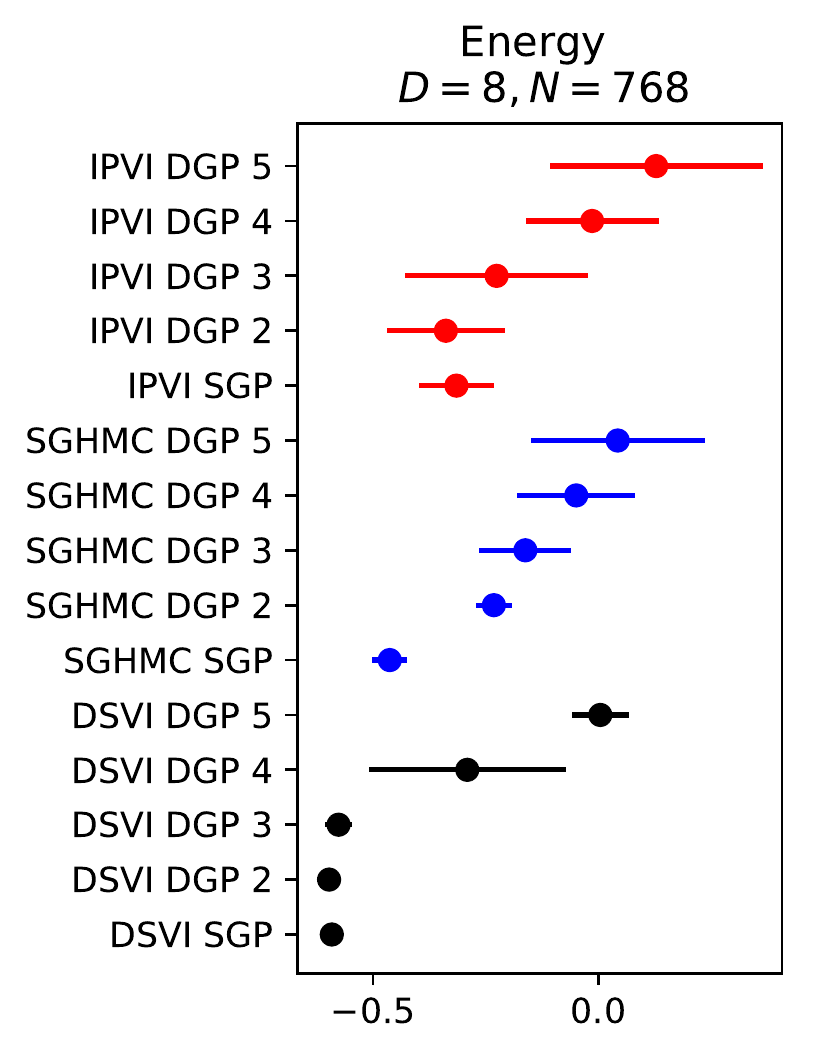}
        & \hspace{-6mm} \includegraphics[height=50mm]{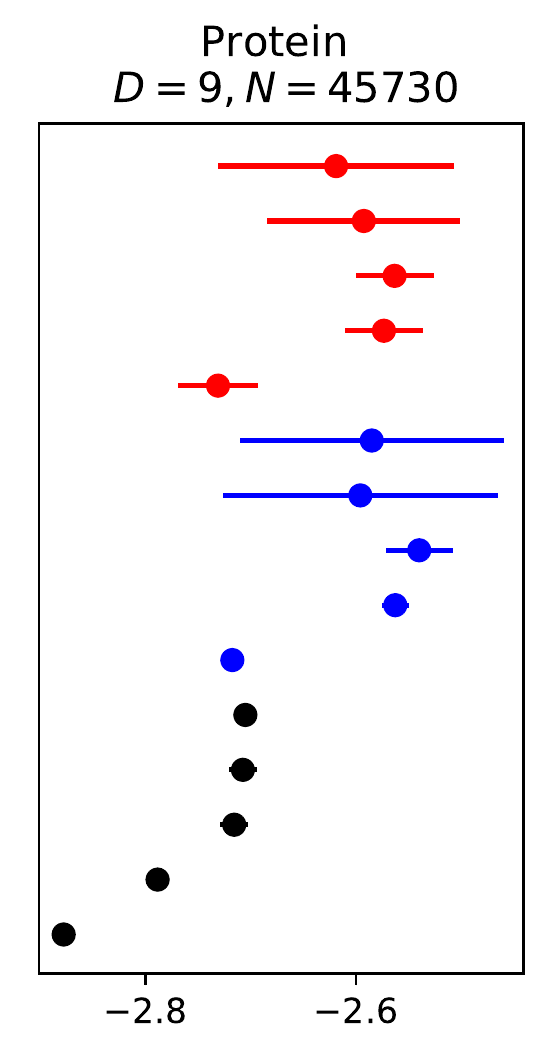}
        & \hspace{-6mm} \includegraphics[height=50mm]{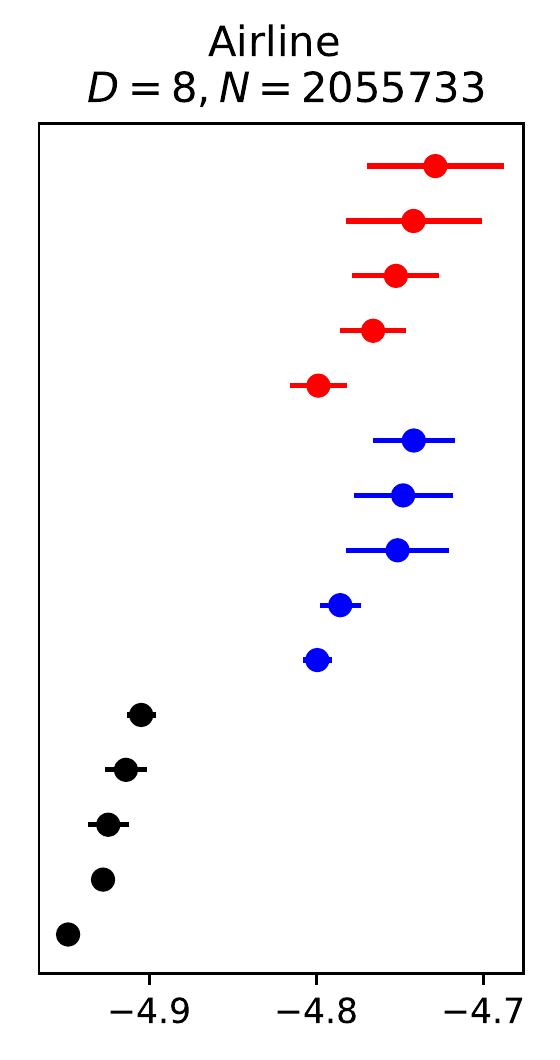}
        & \hspace{-6mm} \includegraphics[height=50mm]{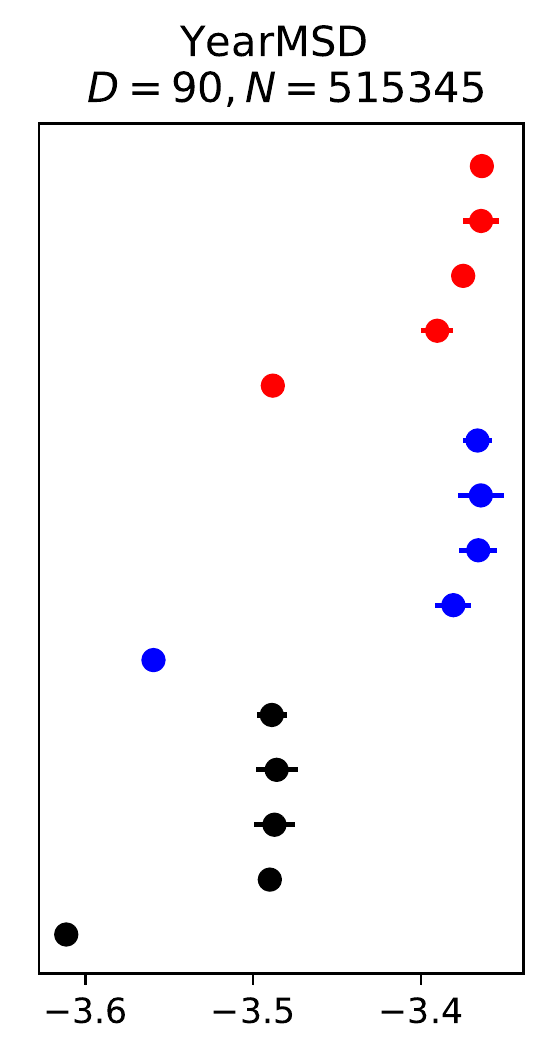}
        & \hspace{-6mm} \includegraphics[height=50mm]{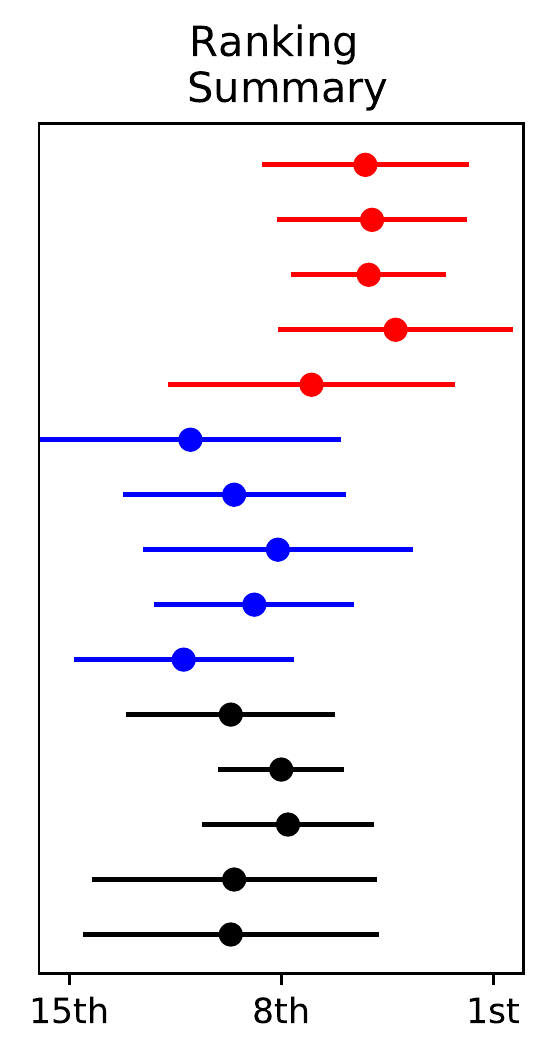}
        \end{tabular}
\caption{Test MLL and standard deviation achieved by our IPVI framework (red), SGHMC (blue), and DSVI (black) for DGPs for UCI benchmark and large-scale regression datasets. Higher test MLL (i.e., to the right) is better. See  Appendix~\ref{crappish} for a discussion on the performance gap between SGPs.}
    \label{fig:uci}
    \end{figure}

    \textbf{Time Efficiency.} Table~\ref{tab: runtime} and Fig.~\ref{fig: runtime} show the better time efficiency of IPVI over the state-of-the-art SGHMC for a $4$-layer DGP model that is trained using the Airline dataset. The learning rates are $0.005$ and $0.02$ for IPVI and SGHMC (default setting adopted from~\cite{havasi2018inference}), respectively.
    Due to parallel sampling (Section~\ref{sec: IPVI DGP}) and a parameter-tying architecture (Section~\ref{sec: arichitecture}), our IPVI framework enables posterior samples to be generated $500$ times faster. Although IPVI has more parameters than SGHMC, it runs $9$ times faster during training due to efficiency in sample generation.
     \begin{figure}
        \begin{minipage}{0.66\linewidth}
            \captionof{table}{Time incurred by a $4$-layer DGP model for Airline dataset.}
            \hspace{3mm}
                        \begin{tabular}{|l|cc|}
                \hline
                & IPVI  & SGHMC   \\ \hline
                Average training time (per iter.) & $0.35$~sec. & $3.18$~sec. \\
                $\boldsymbol{\mathcal{U}}$ generation ($100$ samples) & $0.28$~sec. & $143.7$~sec. \\ \hline
            \end{tabular}
            \label{tab: runtime}
        \end{minipage}
        \hspace{2mm}
        \begin{minipage}{0.31\linewidth}
                \includegraphics[width=43mm]{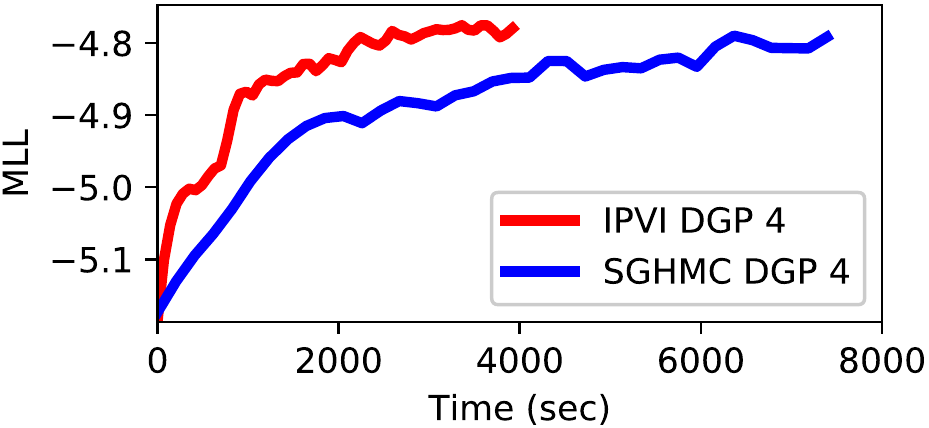}
               \vspace{-5mm}
                \caption{Graph of test MLL vs. total incurred time to train a $4$-layer DGP model for the Airline dataset.}
                \label{fig: runtime}
        \end{minipage}
    \end{figure}            
                 
    \textbf{Classification.} \label{subsec: real world classification} We evaluate the performance of IPVI in three classification tasks using the real-world MNIST, fashion-MNIST, and CIFAR-$10$ datasets. Both MNIST and fashion-MNIST datasets are grey-scale images of $28\times28$ pixels. The CIFAR-$10$ dataset consists of colored images of $32\times32$ pixels. We utilize a $4$-layer DGP model with $100$ inducing inputs per layer and a robust-max multiclass likelihood~\cite{hernandez2011robust}; \textcolor{red}{for MNIST dataset, we also consider utilizing a $4$-layer DGP model with $800$ inducing inputs per layer to assess if its performance improves with more inducing inputs.} Table~\ref{tab: classifications} reports the mean test accuracy over $10$ runs, which shows that our IPVI framework for a $4$-layer DGP model performs the best in all three datasets. Additional results for   IPVI with and without parameter tying are found in Appendix~\ref{crappish}.
        \begin{table}
        \centering
        \caption{\textcolor{red}{Mean test accuracy (\%) achieved by IPVI, SGHMC, and DSVI for $3$ classification datasets.}}
        \begin{tabular}{l|cc|cc|cc|cc}
        \hline
        Dataset     & \multicolumn{2}{c|}{MNIST}
        & \multicolumn{2}{c|}{MNIST ($M=800$)}
        & \multicolumn{2}{c|}{Fashion-MNIST}                & \multicolumn{2}{c}{CIFAR-10}                      \\ \hline
                    & SGP            & DGP 4 
                    & SGP            & DGP 4 
                    & SGP            & DGP 4         & SGP            & DGP 4                     \\ \hline
        DSVI        
        & \textbf{97.32} & 97.41
        & \textbf{97.92} & 98.05
        & 86.98          & 87.99                    & 47.15          & 51.79                    \\
        SGHMC       
        & 96.41          & 97.55
        & 97.07          & 97.91
        & 85.84          & 87.08                    & 47.32          & 52.81                  \\
        \textbf{IPVI} 
        & 97.02          & \textbf{97.80}  
        & 97.85          & \textbf{98.23}  
        & \textbf{87.29} & \textbf{88.90}  & \textbf{48.07} & \textbf{53.27}  \\ \hline
        \end{tabular}      
        \label{tab: classifications}
        \end{table}
%

\subsection{\textcolor{red}{Parameter-Tying vs. No Parameter Tying}} \label{subsec: tying vs no tying}
\textcolor{red}{Table~\ref{tab: parameter tying} reports the train/test MLL achieved by IPVI with and without parameter tying for $2$ small datasets: Boston ($N=506$) and Energy ($N=768$). For Boston dataset, it can be observed that 
no tying consistently yields higher  train MLL and lower test MLL, hence indicating overfitting.
This is also observed for Energy dataset when the number of layers exceeds $2$.
For both datasets, as the number of layers (hence number of parameters) increases, overfitting 
becomes more severe for no tying. In contrast, parameter tying alleviates the overfitting 
considerably.
}

\begin{table}[ht]
\centering
\caption{\textcolor{red}{Train}/test MLL achieved by IPVI with and without parameter tying over $10$ runs.}
\label{tab: parameter tying}
\begin{tabular}{|l|lllll|}
\hline
Dataset    & \multicolumn{5}{c|}{Boston ($N=506$)}                                                                                    \\ \hline
DGP Layers & \multicolumn{1}{c}{1} & \multicolumn{1}{c}{2} & \multicolumn{1}{c}{3} & \multicolumn{1}{c}{4} & \multicolumn{1}{c|}{5} \\
No Tying   & \textcolor{red}{-1.86}/-2.21           & \textcolor{red}{-1.76}/-2.37           & \textcolor{red}{-1.64}/-2.48           & \textcolor{red}{-1.52}/-2.51           & \textcolor{red}{-1.51}/-2.57            \\
Tying      & \textcolor{red}{-1.91}/-2.09           & \textcolor{red}{-1.79}/-2.08           & \textcolor{red}{-1.77}/-2.13           & \textcolor{red}{-1.84}/-2.09           & \textcolor{red}{-1.83}/-2.10            \\ \hline
Dataset    & \multicolumn{5}{c|}{Energy ($N=768$)}                                                                                 \\ \hline
DGP Layers & \multicolumn{1}{c}{1} & \multicolumn{1}{c}{2} & \multicolumn{1}{c}{3} & \multicolumn{1}{c}{4} & \multicolumn{1}{c|}{5} \\
No Tying   & \textcolor{red}{-0.12}/-0.44
           & \textcolor{red}{\ 0.03}/-0.31
           & \textcolor{red}{\ 0.18}/-0.34
           & \textcolor{red}{0.20}/-0.47
           & \textcolor{red}{0.21}/-0.58                 \\
Tying      & \textcolor{red}{-0.16}/-0.32           
           & \textcolor{red}{-0.11}/-0.34           
           & \textcolor{red}{-0.02}/-0.23           
           & \textcolor{red}{0.10}/-0.01           
           & \textcolor{red}{0.17}/\ 0.13                 \\ \hline
\end{tabular}
\end{table}

\subsection{Unsupervised Learning: FreyFace Reconstruction} 
\label{subsec: freyface}
    \begin{figure}
        \includegraphics[scale=0.385]{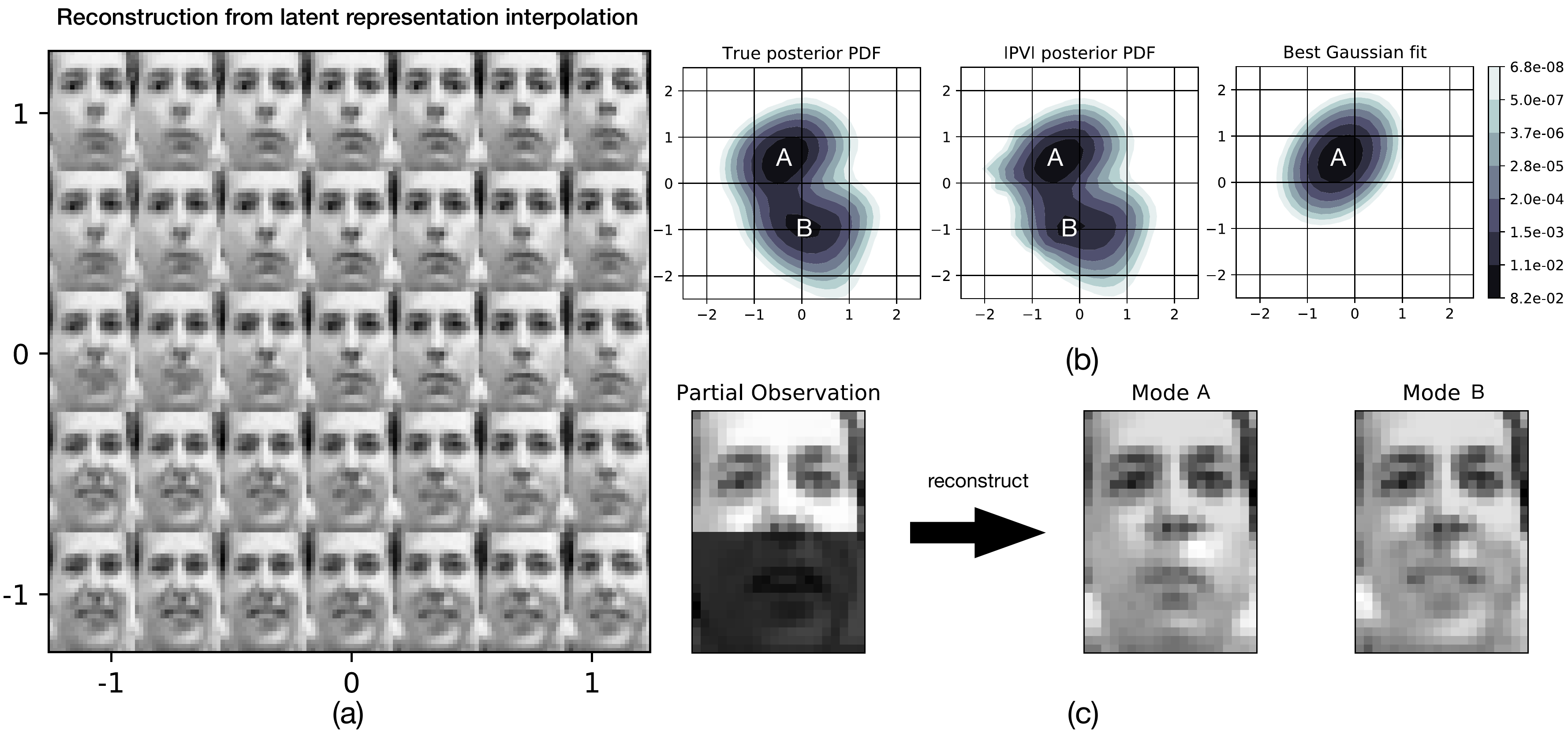}
        \caption{Unsupervised learning with FreyFace dataset. (a) Latent representation interpolation and the corresponding reconstruction. (b) True posterior $p(\mathbf{x}^\star|\mathbf{y}^\star_O)$ given the partial observation $\mathbf{y}^\star_O$ (left), variational posterior $q(\mathbf{x}^\star)$ learned by IPVI (middle), and Gaussian approximation (right). The PDF for $p(\mathbf{x}^\star|\mathbf{y}^\star_O)$ is calculated using Bayes rule where the marginal likelihood is computed using Monte Carlo integration.
(c) The partial observation (with the ground truth reflected in the dark region) and two reconstructed samples from $q(\mathbf{x}^\star)$.}
        \label{fig: gplvm}
    \end{figure}
%
%
%
%
%
A DGP can naturally be generalized to perform unsupervised learning. The representation of a dataset in a low-dimensional manifold can be learned in an unsupervised manner by the \emph{GP latent variable model} (GPLVM)~\cite{lawrence2004gaussian} where only the observations $\mathbf{Y}\triangleq \{\mathbf{y}_n\}^N_{n=1}$ are given and the hidden representation $\mathbf{X}$ is unobserved and treated as latent variables. The objective is to infer the posterior $p(\mathbf{X}|\mathbf{Y})$. The GPLVM is a single-layer GP that casts $\mathbf{X}$ as an unknown distribution and can naturally be extended to a DGP. So, we construct a $2$-layer DGP ($\mathbf{X}\to\mathbf{F}_1\to\mathbf{F}_2\to\mathbf{Y}$) and use the {generator} samples to represent $p(\mathbf{X}|\mathbf{Y})$. 

    We consider the FreyFace dataset~\cite{roweis2002global} taken from a video sequence that consists of $1965$ images with a size of $28 \times 20$. We select the first $1000$ images to train our DGP. To  ease visualization, the dimension of latent variables $\mathbf{X}$ is chosen to be $2$. 
Additional details for our experiments are found in Appendix~\ref{crappier}.    
    Fig.~\ref{fig: gplvm}a shows the reconstruction of faces across the latent space. Interestingly, the first dimension of the latent variables $\mathbf{X}$ determines the expression from happy to calm while the the second dimension controls the view angle of the face. 
    
    We then explore the capability of IPVI in reconstructing partially observed test data. Fig.~\ref{fig: gplvm}b illustrates that given only the upper half of the face, the real face may exhibit a multi-modal property, as reflected in the latent space; intuitively, one cannot always tell whether a person is happy or sad by looking at the upper half of the face. Our variational posterior accurately captures the multi-modal posterior belief whereas the Gaussian approximation can only recover one mode (mode A) under this test scenario. So, IPVI can correctly recover two types of expressions: calm (mode A) and happy (mode B). We did not empirically compare with SGHMC here because it is not obvious to us whether their sampler setting can be carried over to this unsupervised learning task.
\section{Conclusion} 
\label{sec: conclusion}
    This paper describes a novel IPVI framework for DGPs that can ideally recover an unbiased posterior belief of the inducing variables and still preserve the time efficiency of VI. To achieve this, we cast the DGP inference problem as a two-player game and search for a Nash equilibrium (i.e., an unbiased posterior belief) of this game using best-response dynamics.
We propose a novel parameter-tying architecture of the generator and discriminator in our IPVI framework for DGPs to alleviate overfitting 
and speed up training and prediction.  
Empirical evaluation shows that IPVI outperforms the state-of-the-art approximation methods for DGPs in regression and classification tasks and accurately learns complex multi-modal posterior beliefs in our synthetic experiment and an unsupervised learning task.
For future work, we plan to use our IPVI framework for DGPs to accurately represent the belief of the unknown target function in active learning~\cite{LowAAMAS13,NghiaICML14,LowAAMAS12,LowAAMAS08,LowICAPS09,LowAAMAS11,LowAAMAS14,YehongAAAI16} and Bayesian optimization~\cite{dai2019,Erik17,NghiaAAAI18,ling16,yehong2019,yehong17} when the available budget of function evaluations is moderately large.
We also plan to develop 
distributed/decentralized variants~\cite{LowUAI13,LowTASE15,LowRSS13,LowUAI12,hoang2019,HoangICML16,NghiaAAAI19,LowAAAI15,Ruofei18} of IPVI.

\textbf{Acknowledgements.} This research is supported by the National Research Foundation, Prime Minister's Office, Singapore under its Campus for Research Excellence and Technological Enterprise (CREATE) program, Singapore-MIT Alliance for Research and Technology (SMART) Future Urban Mobility (FM) IRG, National Research Foundation Singapore under its AI Singapore Programme Award No. AISG-GC-$2019$-$002$, and the Singapore Ministry of Education Academic Research Fund Tier $2$, MOE$2016$-T$2$-$2$-$156$.

\begin{small}
    \bibliographystyle{abbrv}
    \bibliography{Reference}

\begin{thebibliography}{10}

\bibitem{tf2016}
M.~Abadi, P.~Barham, J.~Chen, Z.~Chen, A.~Davis, J.~Dean, M.~Devin,
  S.~Ghemawat, G.~Irving, M.~Isard, M.~Kudlur, J.~Levenberg, R.~Monga,
  S.~Moore, D.~G. Murray, B.~Steiner, P.~Tucker, V.~Vasudevan, P.~Warden,
  M.~Wicke, Y.~Yu, and X.~Zheng.
\newblock Tensor{F}low: A system for large-scale machine learning.
\newblock In {\em Proc. OSDI}, pages 265--283, 2016.

\bibitem{awerbuch2008fast}
B.~Awerbuch, Y.~Azar, A.~Epstein, V.~S. Mirrokni, and A.~Skopalik.
\newblock Fast convergence to nearly optimal solutions in potential games.
\newblock In {\em Proc. {ACM EC}}, pages 264--273, 2008.

\bibitem{bui2016deep}
T.~Bui, D.~Hern{\'a}ndez-Lobato, J.~Hernandez-Lobato, Y.~Li, and R.~Turner.
\newblock Deep {Gaussian} processes for regression using approximate
  expectation propagation.
\newblock In {\em Proc. {ICML}}, pages 1472--1481, 2016.

\bibitem{LowAAMAS13}
N.~Cao, K.~H. Low, and J.~M. Dolan.
\newblock Multi-robot informative path planning for active sensing of
  environmental phenomena: A tale of two algorithms.
\newblock In {\em Proc. {AAMAS}}, pages 7--14, 2013.

\bibitem{LowUAI13}
J.~Chen, N.~Cao, K.~H. Low, R.~Ouyang, C.~K.-Y. Tan, and P.~Jaillet.
\newblock Parallel {Gaussian} process regression with low-rank covariance
  matrix approximations.
\newblock In {\em Proc. {UAI}}, pages 152--161, 2013.

\bibitem{LowTASE15}
J.~Chen, K.~H. Low, P.~Jaillet, and Y.~Yao.
\newblock Gaussian process decentralized data fusion and active sensing for
  spatiotemporal traffic modeling and prediction in mobility-on-demand systems.
\newblock {\em {IEEE} Transactions on Automation Science and Engineering},
  12(3):901--921, 2015.

\bibitem{LowRSS13}
J.~Chen, K.~H. Low, and C.~K.-Y. Tan.
\newblock {Gaussian} process-based decentralized data fusion and active sensing
  for mobility-on-demand system.
\newblock In {\em Proceedings of the Robotics: Science and Systems Conference
  ({RSS})}, 2013.

\bibitem{LowUAI12}
J.~Chen, K.~H. Low, C.~K.-Y. Tan, A.~Oran, P.~Jaillet, J.~M. Dolan, and G.~S.
  Sukhatme.
\newblock Decentralized data fusion and active sensing with mobile sensors for
  modeling and predicting spatiotemporal traffic phenomena.
\newblock In {\em Proc. {UAI}}, pages 163--173, 2012.

\bibitem{cutajar2017random}
K.~Cutajar, E.~V. Bonilla, P.~Michiardi, and M.~Filippone.
\newblock Random feature expansions for deep {Gaussian} processes.
\newblock In {\em Proc. ICML}, pages 884--893, 2017.

\bibitem{dai2015variational}
Z.~Dai, A.~Damianou, J.~Gonz{\'a}lez, and N.~Lawrence.
\newblock Variational auto-encoded deep {Gaussian} processes.
\newblock In {\em Proc. ICLR}, 2016.

\bibitem{dai2019}
Z.~Dai, H.~Yu, K.~H. Low, and P.~Jaillet.
\newblock Bayesian optimization meets {B}ayesian optimal stopping.
\newblock In {\em Proc. {ICML}}, pages 1496--1506, 2019.

\bibitem{damianou2013deep}
A.~Damianou and N.~Lawrence.
\newblock Deep {Gaussian} processes.
\newblock In {\em Proc. {AISTATS}}, pages 207--215, 2013.

\bibitem{Erik17}
E.~Daxberger and K.~H. Low.
\newblock Distributed batch {Gaussian} process optimization.
\newblock In {\em Proc. {ICML}}, pages 951--960, 2017.

\bibitem{deisenroth2015distributed}
M.~P. Deisenroth and J.~W. Ng.
\newblock Distributed {Gaussian} processes.
\newblock In {\em Proc. ICML}, pages 1481--1490, 2015.

\bibitem{duvenaud2014avoiding}
D.~Duvenaud, O.~Rippel, R.~Adams, and Z.~Ghahramani.
\newblock Avoiding pathologies in very deep networks.
\newblock In {\em Proc. AISTATS}, pages 202--210, 2014.

\bibitem{Yarin14}
Y.~Gal, M.~{van der Wilk}, and C.~E. Rasmussen.
\newblock Distributed variational inference in sparse {Gaussian} process
  regression and latent variable models.
\newblock In {\em Proc. NeurIPS}, pages 3257--3265, 2014.

\bibitem{goodfellow2014generative}
I.~Goodfellow, J.~Pouget-Abadie, M.~Mirza, B.~Xu, D.~Warde-Farley, S.~Ozair,
  A.~Courville, and Y.~Bengio.
\newblock Generative adversarial nets.
\newblock In {\em Proc. NeurIPS}, pages 2672--2680, 2014.

\bibitem{havasi2018inference}
M.~Havasi, J.~M. Hern{\'a}ndez-Lobato, and J.~J. Murillo-Fuentes.
\newblock Inference in deep {Gaussian} processes using stochastic gradient
  {Hamiltonian Monte Carlo}.
\newblock In {\em Proc. NeurIPS}, pages 7517--7527, 2018.

\bibitem{Lawrence13}
J.~Hensman, N.~Fusi, and N.~Lawrence.
\newblock {Gaussian} processes for big data.
\newblock In {\em Proc. {UAI}}, pages 282--290, 2013.

\bibitem{hensman2014nested}
J.~Hensman and N.~D. Lawrence.
\newblock Nested variational compression in deep {Gaussian} processes.
\newblock {arXiv}:1412.1370, 2014.

\bibitem{hernandez2011robust}
D.~Hern{\'a}ndez-Lobato, J.~M. Hern{\'a}ndez-Lobato, and P.~Dupont.
\newblock Robust multi-class {Gaussian} process classification.
\newblock In {\em Proc. NeuIPS}, pages 280--288, 2011.

\bibitem{NghiaAAAI17}
Q.~M. Hoang, T.~N. Hoang, and K.~H. Low.
\newblock A generalized stochastic variational {B}ayesian hyperparameter
  learning framework for sparse spectrum {G}aussian process regression.
\newblock In {\em Proc. {AAAI}}, pages 2007--2014, 2017.

\bibitem{hoang2019}
Q.~M. Hoang, T.~N. Hoang, K.~H. Low, and C.~Kingsford.
\newblock Collective model fusion for multiple black-box experts.
\newblock In {\em Proc. {ICML}}, pages 2742--2750, 2019.

\bibitem{NghiaICML16}
T.~N. Hoang, Q.~M. Hoang, and K.~H. Low.
\newblock A unifying framework of anytime sparse {Gaussian} process regression
  models with stochastic variational inference for big data.
\newblock In {\em Proc. {ICML}}, pages 569--578, 2015.

\bibitem{HoangICML16}
T.~N. Hoang, Q.~M. Hoang, and K.~H. Low.
\newblock A distributed variational inference framework for unifying parallel
  sparse {Gaussian} process regression models.
\newblock In {\em Proc. {ICML}}, pages 382--391, 2016.

\bibitem{NghiaAAAI18}
T.~N. Hoang, Q.~M. Hoang, and K.~H. Low.
\newblock Decentralized high-dimensional {Bayesian} optimization with factor
  graphs.
\newblock In {\em Proc. {AAAI}}, pages 3231--3238, 2018.

\bibitem{NghiaAAAI19}
T.~N. Hoang, Q.~M. Hoang, K.~H. Low, and J.~P. How.
\newblock Collective online learning of {Gaussian} processes in massive
  multi-agent systems.
\newblock In {\em Proc. {AAAI}}, 2019.

\bibitem{NghiaICML14}
T.~N. Hoang, K.~H. Low, P.~Jaillet, and M.~Kankanhalli.
\newblock Nonmyopic $\epsilon$-{B}ayes-optimal active learning of {Gaussian}
  processes.
\newblock In {\em Proc. {ICML}}, pages 739--747, 2014.

\bibitem{hornik1989multilayer}
K.~Hornik, M.~Stinchcombe, and H.~White.
\newblock Multilayer feedforward networks are universal approximators.
\newblock {\em Neural networks}, 2(5):359--366, 1989.

\bibitem{huszar2017variational}
F.~Husz{\'a}r.
\newblock Variational inference using implicit distributions.
\newblock arxiv:1702.08235, 2017.

\bibitem{karras2017progressive}
T.~Karras, T.~Aila, S.~Laine, and J.~Lehtinen.
\newblock Progressive growing of {GANs} for improved quality, stability, and
  variation.
\newblock In {\em Proc. ICLR}, 2018.

\bibitem{kingma2013auto}
D.~P. Kingma and M.~Welling.
\newblock Auto-encoding variational {Bayes}.
\newblock In {\em Proc. ICLR}, 2013.

\bibitem{lawrence2004gaussian}
N.~D. Lawrence.
\newblock {Gaussian} process latent variable models for visualisation of high
  dimensional data.
\newblock In {\em Proc. NeurIPS}, pages 329--336, 2004.

\bibitem{ling16}
C.~K. Ling, K.~H. Low, and P.~Jaillet.
\newblock {Gaussian} process planning with {Lipschitz} continuous reward
  functions: Towards unifying {Bayesian} optimization, active learning, and
  beyond.
\newblock In {\em Proc. {AAAI}}, pages 1860--1866, 2016.

\bibitem{LowAAMAS12}
K.~H. Low, J.~Chen, J.~M. Dolan, S.~Chien, and D.~R. Thompson.
\newblock Decentralized active robotic exploration and mapping for
  probabilistic field classification in environmental sensing.
\newblock In {\em Proc. {AAMAS}}, pages 105--112, 2012.

\bibitem{LowDyDESS15}
K.~H. Low, J.~Chen, T.~N. Hoang, N.~Xu, and P.~Jaillet.
\newblock Recent advances in scaling up {Gaussian} process predictive models
  for large spatiotemporal data.
\newblock In S.~Ravela and A.~Sandu, editors, {\em Proc. Dynamic Data-driven
  Environmental Systems Science Conference ({DyDESS'14})}. LNCS 8964, Springer,
  2015.

\bibitem{LowAAMAS08}
K.~H. Low, J.~M. Dolan, and P.~Khosla.
\newblock Adaptive multi-robot wide-area exploration and mapping.
\newblock In {\em Proc. {AAMAS}}, pages 23--30, 2008.

\bibitem{LowICAPS09}
K.~H. Low, J.~M. Dolan, and P.~Khosla.
\newblock Information-theoretic approach to efficient adaptive path planning
  for mobile robotic environmental sensing.
\newblock In {\em Proc. {ICAPS}}, pages 233--240, 2009.

\bibitem{LowAAMAS11}
K.~H. Low, J.~M. Dolan, and P.~Khosla.
\newblock Active {Markov} information-theoretic path planning for robotic
  environmental sensing.
\newblock In {\em Proc. {AAMAS}}, pages 753--760, 2011.

\bibitem{LowAAAI15}
K.~H. Low, J.~Yu, J.~Chen, and P.~Jaillet.
\newblock Parallel {Gaussian} process regression for big data: Low-rank
  representation meets {M}arkov approximation.
\newblock In {\em Proc. {AAAI}}, pages 2821--2827, 2015.

\bibitem{GPflow2017}
A.~G. d.~G. Matthews, M.~{van der Wilk}, T.~Nickson, K.~Fujii,
  A.~{Boukouvalas}, P.~{Le{\'o}n-Villagr{\'a}}, Z.~Ghahramani, and J.~Hensman.
\newblock {{GP}flow: A {G}aussian process library using {T}ensor{F}low}.
\newblock {\em JMLR}, 18:1--6, 2017.

\bibitem{mescheder2017adversarial}
L.~Mescheder, S.~Nowozin, and A.~Geiger.
\newblock Adversarial variational {Bayes}: Unifying variational autoencoders
  and generative adversarial networks.
\newblock In {\em Proc. {ICML}}, pages 2391--2400, 2017.

\bibitem{Ruofei18}
R.~Ouyang and K.~H. Low.
\newblock Gaussian process decentralized data fusion meets transfer learning in
  large-scale distributed cooperative perception.
\newblock In {\em Proc. AAAI}, pages 3876--3883, 2018.

\bibitem{LowAAMAS14}
R.~Ouyang, K.~H. Low, J.~Chen, and P.~Jaillet.
\newblock Multi-robot active sensing of non-stationary {Gaussian} process-based
  environmental phenomena.
\newblock In {\em Proc. {AAMAS}}, pages 573--580, 2014.

\bibitem{candela05}
J.~{Qui{\~{n}}onero-Candela} and C.~E. Rasmussen.
\newblock A unifying view of sparse approximate {Gaussian} process regression.
\newblock {\em JMLR}, 6:1939--1959, 2005.

\bibitem{Rasmussen06}
C.~E. Rasmussen and C.~K.~I. Williams.
\newblock {\em {Gaussian} processes for machine learning}.
\newblock MIT Press, 2006.

\bibitem{roweis2002global}
S.~T. Roweis, L.~K. Saul, and G.~E. Hinton.
\newblock Global coordination of local linear models.
\newblock In {\em Proc. NeurIPS}, pages 889--896, 2002.

\bibitem{salimbeni2017doubly}
H.~Salimbeni and M.~Deisenroth.
\newblock Doubly stochastic variational inference for deep {Gaussian}
  processes.
\newblock In {\em Proc. NeurIPS}, pages 4588--4599, 2017.

\bibitem{springenberg2016bayesian}
J.~T. Springenberg, A.~Klein, S.~Falkner, and F.~Hutter.
\newblock Bayesian optimization with robust {Bayesian} neural networks.
\newblock In {\em Proc. NeurIPS}, pages 4134--4142, 2016.

\bibitem{Titsias09}
M.~K. Titsias.
\newblock Variational learning of inducing variables in sparse {Gaussian}
  processes.
\newblock In {\em Proc. {AISTATS}}, pages 567--574, 2009.

\bibitem{Titsias09a}
M.~K. Titsias.
\newblock Variational model selection for sparse {Gaussian} process regression.
\newblock Technical report, School of Computer Science, University of
  Manchester, 2009.

\bibitem{titsias2019unbiased}
M.~K. Titsias and F.~J.~R. Ruiz.
\newblock Unbiased implicit variational inference.
\newblock In {\em Proc. AISTATS}, pages 167--176, 2019.

\bibitem{oord2016wavenet}
A.~{van den Oord}, S.~Dieleman, H.~Zen, K.~Simonyan, O.~Vinyals, A.~Graves,
  N.~Kalchbrenner, A.~W. Senior, and K.~Kavukcuoglu.
\newblock Wavenet: A generative model for raw audio.
\newblock {arXiv}:1609.03499, 2016.

\bibitem{wang2018adversarial}
K.-C. Wang, P.~Vicol, J.~Lucas, L.~Gu, R.~Grosse, and R.~Zemel.
\newblock Adversarial distillation of {Bayesian} neural network posteriors.
\newblock In {\em Proc. ICML}, pages 5177--5186, 2018.

\bibitem{LowAAAI14}
N.~Xu, K.~H. Low, J.~Chen, K.~K. Lim, and E.~B. {\"{O}zg\"{u}l}.
\newblock {GP-Localize}: Persistent mobile robot localization using online
  sparse {Gaussian} process observation model.
\newblock In {\em Proc. {AAAI}}, pages 2585--2592, 2014.

\bibitem{yin2018semi}
M.~Yin and M.~Zhou.
\newblock Semi-implicit variational inference.
\newblock In {\em Proc. ICML}, pages 5660--5669, 2018.

\bibitem{Haibin19}
H.~Yu, T.~N. Hoang, K.~H. Low, and P.~Jaillet.
\newblock Stochastic variational inference for {Bayesian} sparse {Gaussian}
  process regression.
\newblock In {\em Proc. {IJCNN}}, 2019.

\bibitem{zhang2019cyclical}
R.~Zhang, C.~Li, J.~Zhang, C.~Chen, and A.~G. Wilson.
\newblock Cyclical stochastic gradient {MCMC} for {Bayesian} deep learning.
\newblock {arXiv}:1902.03932, 2019.

\bibitem{yehong2019}
Y.~Zhang, Z.~Dai, and K.~H. Low.
\newblock Bayesian optimization with binary auxiliary information.
\newblock In {\em Proc. {UAI}}, 2019.

\bibitem{YehongAAAI16}
Y.~Zhang, T.~N. Hoang, K.~H. Low, and M.~Kankanhalli.
\newblock Near-optimal active learning of multi-output {G}aussian processes.
\newblock In {\em Proc. {AAAI}}, pages 2351--2357, 2016.

\bibitem{yehong17}
Y.~Zhang, T.~N. Hoang, K.~H. Low, and M.~Kankanhalli.
\newblock Information-based multi-fidelity {Bayesian} optimization.
\newblock In {\em Proc. {NIPS} Workshop on {Bayesian} Optimization}, 2017.

\end{thebibliography}
\end{small}


\clearpage
\appendix
            
\section{Proof of Proposition~\ref{prop: optimal T(u)}} 
\label{append: proof of Tpsiu}
    The objective function in~\eqref{eq: discriminator} can be re-written as
    \begin{equation*}
        \displaystyle \int p(\boldsymbol{\mathcal{U}}) \log (1 - \sigma(T(\boldsymbol{\mathcal{U}}))) \ \mathrm{d}\boldsymbol{\mathcal{U}} + \int q_{\Phi}(\boldsymbol{\mathcal{U}})\log \sigma(T(\boldsymbol{\mathcal{U}}))\ \mathrm{d}\boldsymbol{\mathcal{U}}\ .
    \end{equation*}
    The above integral is maximal in function $T$ if and only if the integrand is maximal in $T(\boldsymbol{\mathcal{U}})$ for every $\boldsymbol{\mathcal{U}}$. 
    Note that 
       the maximum of $a\log(t) + b\log(1-t)$ over $t\in [0, 1]$ is at $t={a}/{(a+b)}$
    for any $(a, b) \in \mathbb{R}^2 \backslash (0, 0)$.  Using this result, 
    \begin{equation*}
        \sigma(T^*(\boldsymbol{\mathcal{U}})) = \frac{q_{\Phi}(\boldsymbol{\mathcal{U}})} {q_{\Phi}(\boldsymbol{\mathcal{U}}) + p(\boldsymbol{\mathcal{U}})}
    \end{equation*}
    or, equivalently, 
    \begin{equation*}
        T^*(\boldsymbol{\mathcal{U}}) = \log q_{\Phi}(\boldsymbol{\mathcal{U}}) - \log p(\boldsymbol{\mathcal{U}})\ .
    \end{equation*}
\section{Proof of Proposition~\ref{prop: nash equilibrium}} 
\label{append: proof of nash equilibrium}
    If $( \{\Psi^*\}, \{\theta^*, \Phi^*\})$ is a Nash equilibrium, then according to Proposition~\ref{prop: optimal T(u)} and under the assumption that $T_{\Psi^*}$ is expressive enough, we know that \textbf{Player 1} is playing its optimal strategy $\Psi^*$ such that
    \begin{equation}
       T_{\Psi^*}(\boldsymbol{\mathcal{U}}) = \log q_{\Phi^*}(\boldsymbol{\mathcal{U}}) - \log p(\boldsymbol{\mathcal{U}})\ .
       \label{eq: best T}
     \end{equation} 
     Substituting~\eqref{eq: best T} into~\eqref{eq: new elbo} reveals that \textbf{Player 2}'s strategy $\{\theta^*, \Phi^*\}$ maximizes its payoff which is a function of $\{\theta, \Phi\}$:
     \begin{equation}
     \begin{array}{rl}
        \displaystyle \mathcal{F}(\theta, \Phi) \triangleq \hspace{-2.4mm}& \mathbb{E}_{q_{\Phi}(\boldsymbol{\mathcal{U}})}[\mathcal{L}(\theta, \mathbf{X}, \mathbf{y}, \boldsymbol{\mathcal{U}}) + \log p(\boldsymbol{\mathcal{U}}) - \log q_{\Phi^*}(\boldsymbol{\mathcal{U}})] \vspace{1mm}
        \\
        = \hspace{-2.4mm}&\mathbb{E}_{q_{\Phi}(\boldsymbol{\mathcal{U}})}[\mathcal{L}(\theta, \mathbf{X}, \mathbf{y}, \boldsymbol{\mathcal{U}}) + \log p(\boldsymbol{\mathcal{U}}) - \log q_{\Phi}(\boldsymbol{\mathcal{U}})+ \log q_{\Phi}(\boldsymbol{\mathcal{U}})- \log q_{\Phi^*}(\boldsymbol{\mathcal{U}})] \vspace{1mm}
        \\
        =\hspace{-2.4mm}& \mathcal{EL}(\theta, \Phi) + \mathrm{KL}[q_{\Phi}(\boldsymbol{\mathcal{U}})\Vert q_{\Phi^*}(\boldsymbol{\mathcal{U}})]
     \end{array}
     \label{eq: last equation}
     \end{equation}
     where $\mathcal{EL}(\theta, \Phi)$ is the ELBO in~\eqref{eq: phi elbo}. 
     
     Now, suppose that $\{\theta^*, \Phi^*\}$ does not maximize the ELBO. Then, there exists some $\{\theta^\prime, \Phi^\prime\}$ such that $\mathcal{EL}(\theta^\prime, \Phi^\prime) > \mathcal{EL}(\theta^*, \Phi^*)$. By substituting $\{\theta^\prime, \Phi^\prime\}$ into~\eqref{eq: last equation},
     \begin{equation*}
       \mathcal{F}(\theta^\prime, \Phi^\prime) = \mathcal{EL}(\theta^\prime, \Phi^\prime) + \mathrm{KL}[q_{\Phi^\prime}(\boldsymbol{\mathcal{U}})\Vert q_{\Phi^*}(\boldsymbol{\mathcal{U}})] > \mathcal{F}(\theta^*, \Phi^*)\ ,
     \end{equation*}
     which contradicts the fact that $\{\theta^*, \Phi^*\}$ maximizes~\eqref{eq: last equation}. Hence, $\{\theta^*, \Phi^*\}$ maximizes the ELBO, which is equal to the log-marginal likelihood $\log p_{\theta^*}(\mathbf{y})$ with $\theta^*$ being the maximum likelihood assignment and $q_{\Phi^*}(\boldsymbol{\mathcal{U}})$ being equal to the true posterior belief $p(\boldsymbol{\mathcal{U}}|\mathbf{y})$.
\subsection{Discussion on the Existence of Nash Equilibrium}
\label{append: exsist nash}
    \begin{prop}
        Suppose that the parametric representations of $T_\Psi$ and $g_\Phi$ are expressive enough to represent any function and the DGP model hyperparameters are fixed to be $\theta_\circ$.
        Then, the two-player pure-strategy game in~\eqref{eq: payoff} for the case of fixed $\theta_\circ$        
        has a Nash equilibrium. Furthermore, if $(\{\Psi^*\}, \{\theta_\circ,\Phi^*\})$ is a Nash equilibrium,  then $\{\Phi^*\}$ is a global maximizer of the ELBO for the case of fixed $\theta_\circ$ such that $q_{\Phi^*}(\boldsymbol{\mathcal{U}})$ is equal to the true posterior belief $p_{\theta_\circ}(\boldsymbol{\mathcal{U}}|\mathbf{y})$.
    \label{prop: fixed hyper}
    \end{prop}
%
\begin{proof}
    Since we assume the parametric representation of $g_\Phi$ to be expressive enough to represent any function, we can find some $\{\Phi_\circ\}$ such that $q_{\Phi_\circ}(\boldsymbol{\mathcal{U}})$ is equal to the true posterior belief $p_{\theta_\circ}(\boldsymbol{\mathcal{U}}|\mathbf{y})$. We now know that $\{\Phi_\circ\}$ maximizes the ELBO in~\eqref{eq: phi elbo} for the case of fixed DGP model hyperparameters $\theta_\circ$, which we denote by $\mathcal{EL}(\theta_\circ, \Phi_\circ)$. 
    
    Since we assume the parametric representation of $T_\Psi$ to be expressive enough to represent any function, we can further obtain some $\{\Psi_\circ\}$ such that $T_{\Psi_\circ}(\boldsymbol{\mathcal{U}}) = \log q_{\Phi_\circ}(\boldsymbol{\mathcal{U}}) - \log p(\boldsymbol{\mathcal{U}})$. According to Proposition~\ref{prop: optimal T(u)}, $\{\Psi_\circ\}$ maximizes the payoff to \textbf{player 1}. Hence, \textbf{player 1} cannot improve its strategy to achieve a better payoff.
    
    Given that \textbf{player 1} plays strategy $\{\Psi_\circ\}$ for the case of fixed $\theta_\circ$, the payoff to \textbf{player 2} playing  strategy $\{\theta_\circ,\Phi\}$ is
    \begin{equation*}
     \begin{array}{rl}
        \displaystyle \mathcal{F}(\theta_\circ, \Phi) \triangleq \hspace{-2.4mm}&\displaystyle \mathbb{E}_{q_{\Phi}(\boldsymbol{\mathcal{U}})}[\mathcal{L}(\theta_\circ, \mathbf{X}, \mathbf{y}, \boldsymbol{\mathcal{U}}) + \log p(\boldsymbol{\mathcal{U}}) - \log q_{\Phi_\circ}(\boldsymbol{\mathcal{U}})] \vspace{1mm}
        \\
       =\hspace{-2.4mm}&\displaystyle\mathbb{E}_{q_{\Phi}(\boldsymbol{\mathcal{U}})}[\mathcal{L}(\theta_\circ, \mathbf{X}, \mathbf{y}, \boldsymbol{\mathcal{U}}) + \log p(\boldsymbol{\mathcal{U}}) - \log q_{\Phi}(\boldsymbol{\mathcal{U}})+ \log q_{\Phi}(\boldsymbol{\mathcal{U}})- \log q_{\Phi_\circ}(\boldsymbol{\mathcal{U}})] \vspace{1mm}
        \\
       =\hspace{-2.4mm}&\displaystyle \mathcal{EL}(\theta_\circ, \Phi) + \mathrm{KL}[q_{\Phi}(\boldsymbol{\mathcal{U}})\Vert q_{\Phi_\circ}(\boldsymbol{\mathcal{U}})]\vspace{1mm}
        \\
      = \hspace{-2.4mm}&\displaystyle 
        \log p_{\theta_\circ}(\mathbf{y}) -\mathrm{KL}[q_{\Phi}(\boldsymbol{\mathcal{U}})\Vert p_{\theta_\circ}(\boldsymbol{\mathcal{U}}|\mathbf{y})]+ \mathrm{KL}[q_{\Phi}(\boldsymbol{\mathcal{U}})\Vert q_{\Phi_\circ}(\boldsymbol{\mathcal{U}})]\vspace{1mm}
        \\
        =\hspace{-2.4mm}&\displaystyle \log p_{\theta_\circ}(\mathbf{y})\ .
     \end{array}
     \end{equation*}
     So, \textbf{player 2} receives a constant payoff (i.e., independent of $\{\Phi,\theta_\circ\}$) and cannot improve its strategy to achieve a better payoff. Since every player cannot improve strategy to achieve a better payoff, $(\{\Psi_\circ\}, \{ \theta_\circ, \Phi_\circ\})$ is a Nash Equilibrium.

The rest of the proof is similar to that
of Proposition~\ref{prop: nash equilibrium}.
\end{proof}   

\textcolor{red}{Given that the hyperparameters $\theta_\circ$ of a single-layer DGP (i.e., SGP) regression model are fixed, the true posterior belief $p_{\theta_\circ}(\boldsymbol{\mathcal{U}}|\mathbf{y})$ is guaranteed to be a Gaussian \cite{Titsias09a}. In this case, Proposition~\ref{prop: fixed hyper} indicates that $q_{\Phi^*}(\boldsymbol{\mathcal{U}})$ is equal to this Gaussian.
} 
%
%
%
\section{Additional Details for Experiments} 
    \subsection{Synthetic Experiment: Learning a Multi-Modal Posterior Belief} 
    \label{append: toy experiment}
        The prior belief is set as a mixture of $5$ Gaussians:
        \[
            p(\mathbf{f}) \triangleq p_i\sum_{i=1}^5 \mathcal{N}(\mu_i \exp({-8{x^2}}),\mathbf{K}_{\mathbf{XX}})
        \]
        where $p_i\triangleq 1/5$ for $i=1,\ldots,5$, $\mu_1\triangleq -8$, $\mu_2\triangleq -4$, $\mu_3\triangleq 0$, $\mu_4\triangleq 4$, $\mu_5\triangleq 8$, and $\mathbf{K}_{\mathbf{XX}}$ denotes a constant covariance matrix with a constant kernel $k(x,x')\triangleq\sigma_A^2$ and $\sigma_A^2\triangleq 1/(4-\exp(-{8}))$.
      
        Also, $p(\mathbf{y}|\mathbf{f})=\prod_n p(y_n|f_n)=\prod_n (1/(\sqrt{2\pi}\sigma_B))\exp(-(y_i-f_i)^2/(2\sigma_B^2))$ with a large noise variance $\sigma_B^2=7\exp(8)$. Then, the ground-truth posterior belief with $5$ modes can be recovered analytically using Bayes rule:
        \[
        p(\mathbf{f}|\mathbf{y}) = p'_i\sum_{i=1}^5 \mathcal{N}(\mu_i \exp(-8{x^2})+\delta_i,\mathbf{K}'_{\mathbf{XX}})
        \]
        where $p'_1= 0.1988$, $p'_2= 0.2004$, $p'_3= 0.2016$, $p'_4= 0.2004$, $p'_5= 0.1988$, $\delta_1= 0.000479$, $\delta_2= 0.00024$, $\delta_3= 0$, $\delta_4= -0.00024$, $\delta_5= -0.000479$,
        and $\mathbf{K}'_{\mathbf{XX}}$ denotes a constant covariance matrix with a constant kernel  $k'(x,x')\triangleq\sigma_C^2$ and
        $\sigma_C^2=1/4$. 
        
        In our implementation, the ground-truth GP kernel hyperparameter values are known to IPVI and SGHMC. We adopt a single inducing input fixed at $z=0$. The multi-modal posterior belief $p(f|\mathbf{y};x=0)$ is then approximated using the samples from $p(u|\mathbf{y};z=0)$. In 
Fig.~\ref{fig: SGHMC different setting}, we give additional results for different hyperparameter settings of   SGHMC to show that it is likely to obtain a biased posterior belief.
        
        We vary the number of hidden layers and number of neurons in each hidden layer to obtain generators with different number of parameters in Fig.~\ref{fig: synthetic data}c.
        \begin{figure}[ht]
        \begin{tabular}{lll}
          \includegraphics[height=33mm]{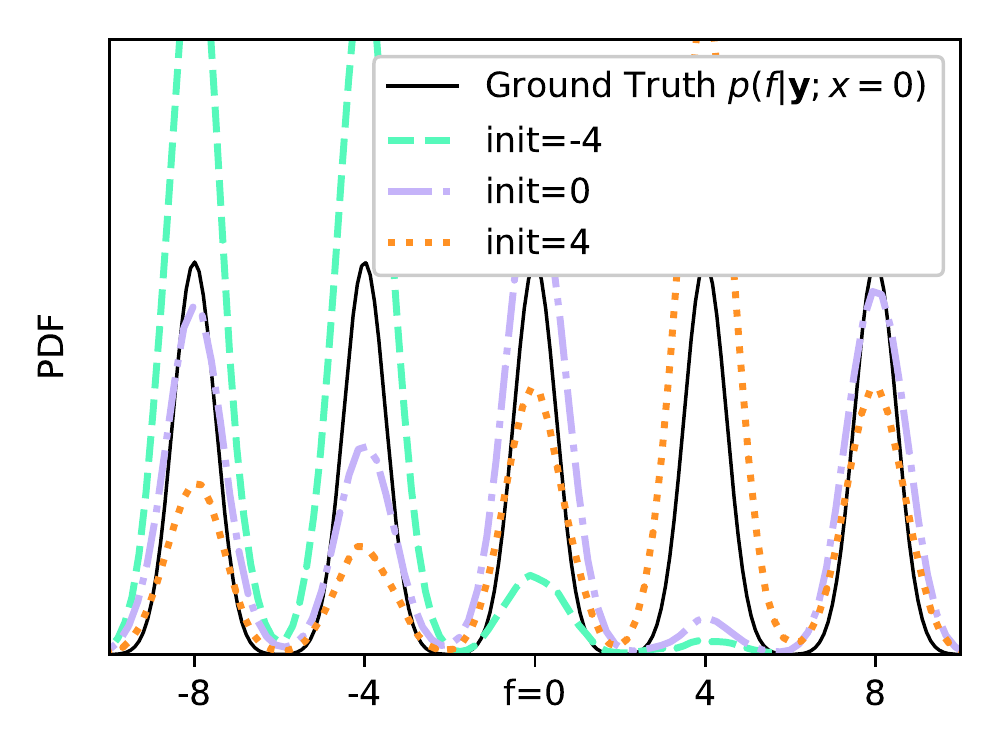} 
          &
          \hspace{-5mm}\includegraphics[height=33mm]{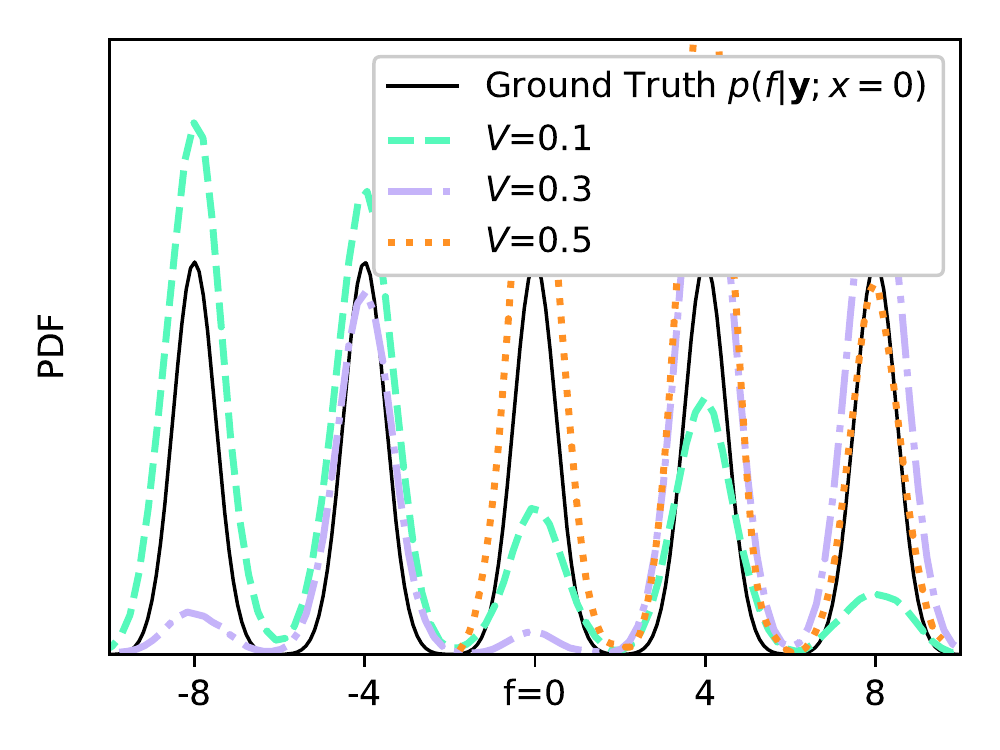} 
          &
          \hspace{-5mm}\includegraphics[height=33mm]{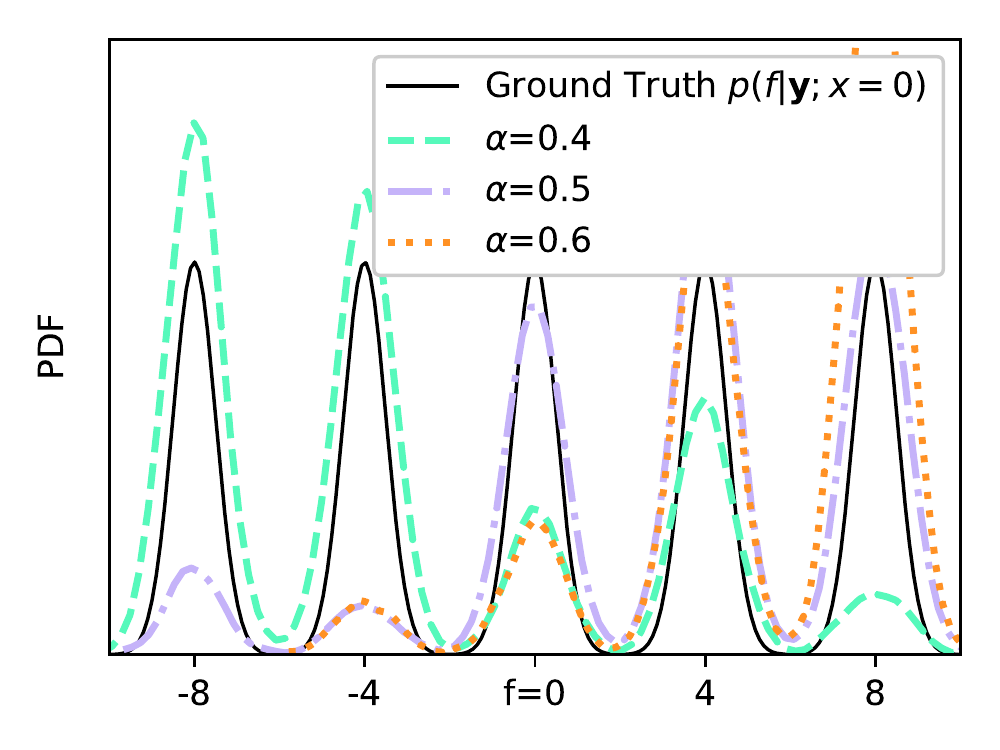}
          \\
          \hspace{22mm}(a) & \hspace{17mm}(b) & \hspace{17mm}(c) 
        \end{tabular}
  \caption{SGHMC with different hyperparameter settings of learning rate $\eta$, momentum $1-\alpha$, Fisher information $V$, and initialization $\mathrm{init}$ for starting the sampler: (a) $\eta=0.3, \alpha=0.4, V=0.1$; (b) $\eta=0.3, \mathrm{init}=4, \alpha=0.4$; and (c)  $\eta=0.3, \mathrm{init}=4, V=0.1$.} 
        \label{fig: SGHMC different setting}  
  \end{figure} 

\subsection{Unsupervised Learning: FreyFace Reconstruction}
\label{crappier}
    The dimensions of the hidden layers are $2$ for $\mathbf{X}$ and $100$ for $\mathbf{F}_1$ for FreyFace Reconstruction. We did not exploit inducing variables here. So, the training is a full DGP. We use PCA as the mean function for this unsupervised learning task.
    
    \textbf{Reconstruction.} Given a trained DGP model, the {reconstruction} task of a partially observed $\mathbf{y}_O^\star$ is to recover the missing part $\mathbf{y}_U^\star$ such that $\mathbf{y}^\star=[\mathbf{y}_O^\star,\mathbf{y}_U^\star]$. This reconstruction task involves two steps. The first step is to cast it as an DGP inference problem to get the posterior $p(\mathbf{x}^\star|\mathbf{y}_O^\star)$ with a Gaussian likelihood $p(\mathbf{y}_O^\star|\mathbf{y}^\star)$. The second step samples $\mathbf{y}^\star$ from $p(\mathbf{y}^\star|\mathbf{y}^\star_O)=\int p(\mathbf{y}^\star|\mathbf{x}^\star)\ p(\mathbf{x}^\star|\mathbf{y}_O^\star)\ \mathrm{d}\mathbf{x}^\star$. 

\subsection{Supervised Learning: Regression and Classification}
\label{crappish}
    In this subsection, we provide additional details for our experiments in the supervised learning tasks.
    
    \textbf{Learning Rates.} We adopt the default settings of the learning rates of the tested methods from   their publicly available implementations. The learning rates and maximum iteration for IPVI are tuned through grid search and cross validation with a default setting of $\alpha_\Psi=0.05$, $\alpha_\Phi=0.001$, $\alpha_\theta=0.025$ and cut-off at a maximum of $20000$ iterations. The learning rates for classification is simply set to be $0.02$ for all parameters.
    
    \textbf{Hidden Dimensions.} The dimension of inducing variables for all implementations are set to be  (i) the same as input dimension for the UCI benchmark regression and Airline datasets, (ii) $16$ for the YearMSD dataset, and (iii) $98$ for the classification tasks.
    
    \textbf{Mini-Batch Sizes.} The mini-batch sizes for all implementations are set to be (i) $10000$ for the UCI benchmark regression tasks, (ii) $20000$ for the large-scale regression tasks, and (iii) $256$ for the classification tasks.
    
    \textbf{Generator/Discriminator Details.} We have described the architecture design in Section~\ref{sec: arichitecture}. We will describe here the neural network represented by $g_{\phi_\ell}$. Firstly, the noise $\epsilon$ has the same dimension as the inputs $\mathbf{X}$ of the dataset. We implement $g_{\phi_\ell}$ using a two-layer neural network with hidden dimension being equal to the dimension of $\mathbf{Z}_{\ell}$ and leaky ReLU activation in the middle. Similarly, we implement $T_{\psi_\ell}$ using a two-layer neural network with hidden dimension being equal to the dimension of $\mathbf{Z}_{\ell}$ and leaky ReLU activation in the middle. The network initialization follows random normal distribution.
    
    \textbf{Mean Function of DGP.} The `skip-layer' connections are implemented in both SGHMC~\cite{havasi2018inference} and DSVI~\cite{salimbeni2017doubly} for DGPs and in our IPVI framework as well. The work of~\cite{duvenaud2014avoiding} has analyzed that using a zero mean function in the DGP prior causes some difficulty as each GP mapping is highly non-injective. To mitigate this issue, the work of~\cite{salimbeni2017doubly} has proposed to include a linear mean function $m(\mathbf{X}) = \mathbf{W}\mathbf{X}$ for all hidden layers. The 'skip-layer' connection $\mathbf{W}$ is set to be an identity matrix if the input dimension equals to the output dimension. Otherwise, $\mathbf{W}$ is computed from the top $H$ eigenvectors of the data under SVD. We follow the same setting as this 'skip-layer' mean function. Note that this 'skip-layer' mean function contains no trainable parameters.
    
    \textbf{Likelihood.} For the classification tasks, we use the {robust-max multiclass likelihood} \cite{hernandez2011robust}. Tricks like data augmentation are not applied, which means that the accuracy can still be improved further with those additional tricks.

\textbf{Parameter-Tying vs. No Parameter-Tying.} Tables~\ref{tab: tying regression} and~\ref{tab: tying classification} show, respectively, results of the test MLL for more UCI benchmark regression datasets and the mean test accuracy for the three classification tasks over $10$ runs that are achieved by IPVI with and without parameter tying. It can be observed that IPVI achieves a considerably better predictive performance with parameter tying. 
%

\begin{table}[ht]
    \caption{Test MLL achieved by our IPVI framework with and without parameter tying  for UCI benchmark regression datasets. Higher test MLL is better.}
    \begin{tabular}{|l|lllll|lllll|}
    \hline
    Dataset    & \multicolumn{5}{c|}{Boston}                                                                                            & \multicolumn{5}{c|}{Power}                                                                                             \\ \hline
    DGP Layers & \multicolumn{1}{c}{1} & \multicolumn{1}{c}{2} & \multicolumn{1}{c}{3} & \multicolumn{1}{c}{4} & \multicolumn{1}{c|}{5} & \multicolumn{1}{c}{1} & \multicolumn{1}{c}{2} & \multicolumn{1}{c}{3} & \multicolumn{1}{c}{4} & \multicolumn{1}{c|}{5} \\
    No Tying  & -2.21                 & -2.37                 & -2.48                 & -2.51                 & -2.57                  & -2.77                 & -2.79                 & -2.74                 & -2.73                 & -2.75                  \\
    Tying      & -2.09                 & -2.08                 & -2.13                 & -2.09                 & -2.10                  & -2.76                 & -2.69                 & -2.67                 & -2.70                 & -2.71                  \\ \hline
    Dataset    & \multicolumn{5}{c|}{Wine Red}                                                                                          & \multicolumn{5}{c|}{Protein}                                                                                           \\ \hline
    DGP Layers & \multicolumn{1}{c}{1} & \multicolumn{1}{c}{2} & \multicolumn{1}{c}{3} & \multicolumn{1}{c}{4} & \multicolumn{1}{c|}{5} & \multicolumn{1}{c}{1} & \multicolumn{1}{c}{2} & \multicolumn{1}{c}{3} & \multicolumn{1}{c}{4} & \multicolumn{1}{c|}{5} \\
    No Tying  & -0.97                 & -0.94                 & -0.96                 & -0.97                 & -0.97                  & -2.83                 & -2.72                 & -2.69                 & -2.70                 & -2.67                  \\
    Tying      & -0.84                 & -0.81                 & -0.86                 & -0.86                 & -0.85                  & -2.73                 & -2.57                 & -2.56                 & -2.59                 & -2.62                  \\ \hline
    \end{tabular}
    \label{tab: tying regression}
    \end{table}
\begin{table}[ht]
    \centering
    \caption{Mean test accuracy (\%) achieved by our IPVI framework with and without parameter tying   for three classification datasets.}
    \begin{tabular}{|l|ll|ll|ll|}
    \hline
    Dataset    & \multicolumn{2}{c|}{MNIST}                     & \multicolumn{2}{c|}{fashion-MNIST}             & \multicolumn{2}{c|}{CIFAR-$10$}                   \\ \hline
    DGP Layers & \multicolumn{1}{c}{1} & \multicolumn{1}{c|}{4} & \multicolumn{1}{c}{1} & \multicolumn{1}{c|}{4} & \multicolumn{1}{c}{1} & \multicolumn{1}{c|}{4} \\
    No Tying  & 96.77                 & 97.45                  & 86.69                 & 88.01                  & 47.13                 & 52.76                  \\
    Tying      & 97.02                 & 97.80                  & 87.29                 & 88.90                  & 48.07                 & 53.27                  \\ \hline
    \end{tabular}
    \label{tab: tying classification}
    \end{table}

\textbf{Performance Gap between SGPs.} Regarding the performance gap between SGPs, note that the optimal variational posterior is a Gaussian for a SGP regression model~\cite{Titsias09a}. 
However, since the SGP model hyperparameters are not known beforehand, DSVI SGP has to jointly optimize its hyperparameters and variational parameters. Such an optimization is not convex. Hence, there is no guarantee that it will reach the global optimum. Thus, the performance gap can be explained by IPVI’s ability to jointly find ``better'' values of hyperparameters and variational parameters.

\textbf{Evaluation of ELBO.} \textcolor{red}{We have also computed the estimate of ELBO by, after training our IPVI DGP models for the Boston dataset, continuing to train the discriminator using more calls of Algorithm~\ref{alg: D}. Table~\ref{tab: elbo} shows the mean ELBOs of DSVI and IPVI over $10$ runs for the Boston dataset. IPVI generally achieves higher ELBOs, which agrees with results of the test MLL in Fig.~\ref{fig:uci}. Since SGHMC DGP is not based on VI, no ELBO is computed for that method.}

\begin{table}[ht]
    \centering
    \caption{Mean ELBOs for Boston dataset.}
    \label{tab: elbo}
    \begin{tabular}{|c|cc|}
    \hline
    Model & \multicolumn{1}{c}{DSVI} & IPVI \\ \hline
    SGP            & {-956.57}                        & {-934.07}    \\
    DGP 2          & {-850.54}                        & {-846.65}    \\
    DGP 3          & {-836.13}                        & {-846.45}    \\
    DGP 4          & {-787.10}                        & {-776.93}    \\
    DGP 5          & {-770.67}                        & {-758.42}    \\ \hline
    \end{tabular}
\end{table}

\end{document}